\documentclass{article}

\usepackage[preprint]{neurips_2026}
\usepackage{microtype}
\usepackage{graphicx}
\usepackage{subcaption}
\usepackage{booktabs} 
\usepackage{amsmath}
\usepackage{amssymb}
\usepackage{mathtools}
\usepackage{amsthm}
\usepackage{hyperref}
\hypersetup{
    breaklinks,
    colorlinks,
    citecolor=blue
}

\newcommand{\slabel}[1]{%
  \textcolor{orange!80!black}{\textbf{[#1]}}}
  \newcommand{\RETURN}{\STATE \textbf{return} }
\usepackage{algorithm}
\usepackage[capitalize,noabbrev]{cleveref}

\theoremstyle{plain}
\newtheorem{theorem}{Theorem}[section]
\newtheorem{proposition}{Proposition}[section]
\newtheorem{lemma}{Lemma}[section]

\theoremstyle{definition}

\newtheorem{assumption}{Assumption}[section]
\theoremstyle{remark}
\newtheorem{remark}{Remark}[section]

\usepackage[utf8]{inputenc} 
\usepackage[T1]{fontenc}    
\usepackage{hyperref}       
\usepackage{url}            
\usepackage{booktabs}       
\usepackage{amsfonts}       
\usepackage{nicefrac}       
\usepackage{microtype}      
\usepackage{xcolor}         
\usepackage{booktabs}
\usepackage{multirow}
\usepackage{algorithmic}
\usepackage[textsize=tiny]{todonotes}

\title{When Muon Optimizer Meets Adversarial Training:\\A Theoretical and Empirical Study}

\author{%
\begin{tabular}{@{}ccc@{}}
Jun Yan$^{1}$ & Weiquan Huang$^{2}$ & Jiankai Zuo$^{3}$ \\
\texttt{yanjun@ieee.org} &
\texttt{weiquanh@tongji.edu.cn} &
\texttt{zuojiankai@usts.edu.cn} \\[0.8em]
Yujian Mo$^{2}$ & Xi Fang$^{4}$ & Chengliang Wu$^{1}$ \\
\texttt{yujmo@tongji.edu.cn} &
\texttt{xifang@dp.tech} &
\texttt{clwu@shou.edu.cn} \\[0.8em]
\multicolumn{3}{c}{Zeming Wei$^{5}$} \\
\multicolumn{3}{c}{\texttt{weizeming@stu.pku.edu.cn}} \\[1.0em]
\multicolumn{3}{c}{\small $^{1}$IT College, Shanghai Ocean University, Shanghai, 201306} \\
\multicolumn{3}{c}{\small $^{2}$School of Computer Science and Technology, Tongji University, Shanghai, 201804} \\
\multicolumn{3}{c}{\small $^{3}$SEIE, Suzhou University of Science and Technology, Suzhou, 215009} \\
\multicolumn{3}{c}{\small $^{4}$DP Tech, Shanghai, 200030} \\
\multicolumn{3}{c}{\small $^{5}$School of Mathematical Sciences, Peking University, Beijing, 100871}
\end{tabular}%
}

\begin{document}

\maketitle

\begin{abstract}
Adversarial training (AT) remains one of the most reliable empirical defenses against adversarial attacks. Its robustness critically depends on how the underlying min-max objective is optimized. In practice, Stochastic Gradient Descent (SGD) optimizer remains the default optimization choice for AT, whereas adaptive optimizers often improve standard training but may yield inferior robustness. Recently, the Muon optimizer, which orthogonalizes matrix-valued updates via an approximate polar decomposition, has achieved notable success in large-scale training at a memory cost comparable to SGD. This raises a security-relevant question: \textit{can orthogonalized optimization improve AT under strong and heterogeneous threat models?} Focusing on this problem, we conduct a comprehensive theoretical and empirical study. Theoretically, we show that Muon imposes a spectral-norm stability ceiling on matrix updates, limiting uncontrolled spectral growth in the training dynamics without explicitly shrinking the learned weights. Empirically, across five architectures and three $\ell_p$ threat models ($\ell_\infty$, $\ell_1$, $\ell_2$) and their union, Muon is competitive with SGD on CNNs and substantially outperforms AdamW on both CNNs and ViTs. These results identify optimizer geometry as a security-relevant factor in adversarial training, while clarifying the empirical regimes in which orthogonalized updates are beneficial. Overall, our findings highlight optimizer design as a security-critical component of AT.
\end{abstract}

\section{Introduction}
\par Deep neural networks are increasingly deployed in security-sensitive settings, yet their predictions can be manipulated by small, human-imperceptible adversarial perturbations~\cite{szegedy2014intriguing,goodfellow2015explaining, madry2018towards,carlini2017towards}. To mitigate such potential threats to neural networks, robust training methods known as adversarial training (AT) have been proposed~\cite{madry2018towards,zhang2019theoretically,wong2020fast,croce2022adversarial}. These methods are based on robust optimization, aiming to maximize adversarial perturbations while minimizing empirical risk. 

\par The outer optimization problem in AT is often noisier and more ill-conditioned than standard empirical risk minimization, because each update is computed on adversarially perturbed inputs produced by an inner maximization procedure. In this regime, optimizer-induced geometry can have a non-negligible effect on both clean accuracy and robust accuracy. Adaptive optimizers such as Adam and AdamW rescale each coordinate of the update by an estimate of its gradient second moment. The Stochastic Gradient Descent (SGD) optimizer and its momentum variants avoid this per-coordinate rescaling, and have been observed to yield greater robustness under the high-variance gradients of AT~\cite{dabouei2022revisiting}. Empirical studies have also validated their superiority in AT~\cite{pang2021bag}. Some studies have also focused on enhancing robustness by exploring the Sharpness-Aware Minimization (SAM) method, which perturbs model weights to seek a flatter loss landscape~\cite{zhang2024duality,zhou2025sharpness}. These studies make the parameter dynamics of min-max optimization analytically tractable, allowing robustness gains to be attributed to specific update-rule properties (e.g., implicit regularization of curvature or spectral norms) rather than treated as emergent artifacts.
\par A previous study~\cite{pang2021bag} found that the effectiveness of using Adam~\cite{kingma2015adam} or AdamW~\cite{loshchilov2017decoupled} in AT is suboptimal. Recently, the Muon optimizer~\cite{jordan2024muon} has been introduced, which avoids expensive high-order projection algorithms while exploiting orthogonalized update geometry and enables a large-batch-friendly update that can be efficiently approximated by a small fixed number of Newton-Schulz iterations. This optimization method constrains the update direction (the polar factor of the momentum-accumulated gradient) to the Stiefel manifold, yielding orthogonalized, spectrally balanced parameter updates (though the weights $W_t$ themselves need not lie on the manifold)~\cite{liu2025muon}. The Muon optimizer differs from other optimizers in that it orthogonalizes matrix-valued updates. For each 2D weight tensor, the Muon optimizer transforms the momentum update into its approximate polar factor, which can be realized via Newton–Schulz iterations. The updated geometry is fundamentally different from SGD~\cite{robbins1951stochastic,bottou2012stochastic} and AdamW~\cite{loshchilov2017decoupled}. However, the impact of this optimizer on adversarial robustness remains unexplored.
\par To fill this research gap, we carry out a systematic theoretical and experimental study of the Muon optimizer in the AT paradigm, as illustrated in Figure~\ref{fig:framework}. We first establish a theoretical mechanism and explanation linked to robustness. Intuitively, adversarial vulnerability is governed by norms of input gradients and Jacobians~\cite{jakubovitz2018improving}. We show that Muon's polar update induces a spectral-norm stability ceiling at the update level. Each matrix-valued update has a controlled operator norm after orthogonalization. This property does not directly enforce small spectral norms of the learned weights. Instead, it provides a mechanism for limiting abrupt spectral growth in the optimization trajectory, offering a diagnostic explanation for the robustness behavior observed in our experiments. Empirically, we first adapt the Muon optimizer for AT, then conduct comprehensive experiments across multiple datasets and model architectures, to evaluate the robustness of four widely used optimizers: SGD, AdamW, SAM, and Muon. 
Across this grid, we observe that optimizer choice can induce large and consistent shifts in the robustness frontier: the same training recipe can yield markedly different robustness profiles depending on optimization dynamics. In particular, the robustness gains provided by the Muon optimizer are significantly greater than those of AdamW,  and Muon's performance is competitive with SGD across multiple settings. Compared with SAM, Muon should be viewed as an optimizer-geometry alternative rather than a direct replacement. It improves update geometry without introducing an additional weight-perturbation step, but it does not uniformly dominate SAM in robustness. To summarize, our key empirical findings are: \textbf{(1)} Muon constrains the update geometry: each matrix-valued update has a bounded operator norm after polar orthogonalization. This does not imply that Muon always produces smaller weight spectral norms than SGD, but it helps prevent the unstable spectral growth often observed with adaptive optimizers such as AdamW. \textbf{(2)} On small-scale datasets like CIFAR-10~\cite{krizhevsky2009learning}, Muon is on par with or superior to SGD on overparameterized CNNs~\cite{zagoruyko2016wide}. On ViTs~\cite{dosovitskiy2021image}, Muon lags behind SGD yet remains markedly more stable and/or outperforms AdamW. However, on large-scale datasets like ImageNet~\cite{deng2009imagenet}, SGD still performs better, suggesting that sensitivity to large-scale learning rates is a key bottleneck, deserving future investigations. \textbf{(3)} Muon is an efficient optimizer-geometry alternative rather than a strict replacement for SAM.
\begin{figure}[!t]
    \centering
    \includegraphics[width=0.85\hsize]{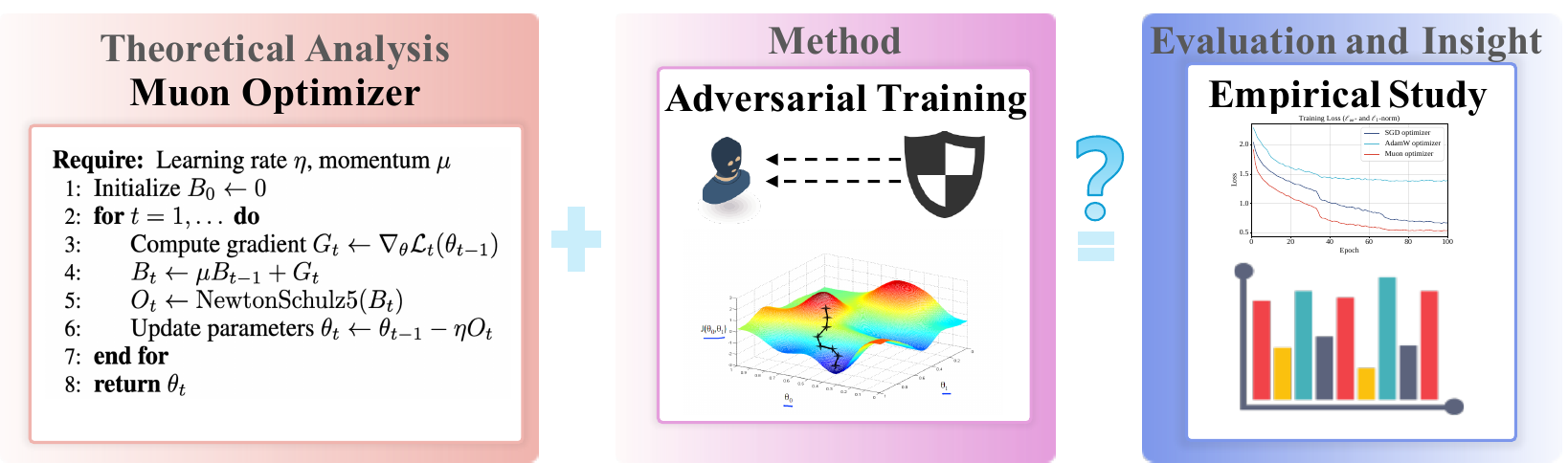}
    \caption{The framework of this study. The left part of the figure is adapted from the work~\cite{jordan2024muon}.}
    \label{fig:framework}
    \vspace{-10pt}
\end{figure}
\par Our contributions are summarized as follows:
\begin{itemize}
\item \textbf{Revisiting optimizer choice as a security-relevant factor in AT.}
We identify orthogonalized matrix-update geometry as an underexplored factor in adversarial training and study Muon as a representative optimizer beyond the conventional SGD/AdamW/SAM comparison.
\item \textbf{Mechanistic analysis of Muon's spectral update geometry.}
We analyze Muon's polar update rule and show that it imposes an operator-norm ceiling on matrix-valued updates. We further relate this update-level constraint to spectral growth and nuclear-norm descent under adversarial loss dynamics, clarifying both what Muon can and cannot guarantee.
\item \textbf{Controlled multi-norm robustness evaluation.} We conduct a controlled empirical comparison of SGD, AdamW, and Muon across multiple architectures, datasets, and threat norms. We further include representative SAM comparisons on WideResNet-34-10 (WRN-34-10)~\cite{zagoruyko2016wide} and ViT-B~\cite{dosovitskiy2020an} to clarify the trade-off between orthogonalized updates and sharpness-aware perturbation-based optimization. Our results show that Muon often improves over AdamW and is competitive with SGD in several small-scale settings, while also exposing its scaling limitations on larger datasets.
\item \textbf{Threat model and scope}. We consider standard test-time adversarial robustness under white-box $\ell_p$-bounded attacks. The attacker has full access to the deployed classifier and its gradients, but the proposed method changes only the training-time optimizer. At inference time, the neural networks do not require gradient masking to increase robustness. We evaluate black box migration attacks to adapt to real-world security risks.
\end{itemize}

\section{Theoretical Analysis}
\subsection{Preliminaries}
\label{sec:problem}
\textbf{Adversarial Training (AT).} Let $f_\theta: \mathcal{X} \to \mathbb{R}^K$
denote a classifier parameterized by $\theta$, with logits $g_\theta(x)$ and a classification margin $m_\theta(x, y) := g_\theta(x)_y - \max_{y' \ne y} g_\theta(x)_{y'}$.
AT solves the min-max problem
\begin{equation}
\min_\theta \mathbb{E}_{(x, y) \sim \mathcal{D}}
\left[ \max_{\delta \in \Delta} \ell(f_\theta(x + \delta), y) \right],
\label{eq:at-objective}
\end{equation}
where $\ell(\cdot, \cdot)$ is the classification loss.

\textbf{Threat model.} We consider the union of multiple $\ell_p$ threat models.
Given a set of norms $\mathcal{P} \subseteq \{1, 2, \infty\}$ with per-norm
radii $\{\varepsilon_p\}_{p \in \mathcal{P}}$, the perturbation set is
\begin{equation}
\Delta_\mathrm{union}
:= \bigcup_{p \in \mathcal{P}} \{\delta \in \mathbb{R}^d : \|\delta\|_p \le \varepsilon_p\}.
\end{equation}

\textbf{Muon update rule.} For each 2D weight block $W_t\in\mathbb{R}^{m\times n}$,
let $G_t = \nabla_W \mathcal{L}_{\mathrm{adv}}(\theta_{t-1})$ denote the
AT gradient and $M_t = \mu M_{t-1} + G_t$ the momentum-accumulated gradient.
Muon applies the orthogonalized update
\begin{equation}
U_t=\operatorname{polar}(M_t),\qquad
W_{t+1}=(1-\eta\lambda)W_t-\eta U_t,
\label{eq:muon-update}
\end{equation}
where $\eta > 0$ is the learning rate, $\lambda \ge 0$ is weight decay, and
$\mathrm{polar}(\cdot)$ is approximated by a fixed number of Newton--Schulz
iterations (see Algorithm~\ref{alg:muon-step}).

\textbf{Notation.} We use $\|\cdot\|_2$ for the spectral (operator-2) norm,
$\|\cdot\|_F$ for the Frobenius norm, $\|\cdot\|_*$ for the nuclear norm,
and $\sigma_i(\cdot)$ for singular values. $J_f(x)$ denotes the Jacobian of
$f$ at $x$. Per-layer weights are $W_\ell$ with spectral norm upper bound
$B_\ell$ (to be established in Theorem~\ref{thm:upper_bound_spectral_norm}). We analyze how the update rule
of Eq. \eqref{eq:muon-update} governs $\|W_\ell\|_2$, the Lipschitz constant of
$m_\theta$, and the per-step descent on the AT objective
defined in Eq. \eqref{eq:at-objective}.
\subsection{Spectral-Norm Stability Ceiling}
\label{subsec:spectral_stability_ceiling}
\par This section proves that Muon's polar update admits a spectral-norm stability ceiling (Theorem~\ref{thm:upper_bound_spectral_norm}). Theoretically, we show that Muon's polar update imposes a spectral-norm stability ceiling on each matrix update, preventing the unbounded spectral growth seen with adaptive optimizers. It is a structural property rather than a tight robustness certificate.
\begin{proposition}[Robustness governed by the dual norm of input gradients]
\label{prop:dual_gradients}
Under a first-order Taylor expansion of $m_\theta$,
\[
m_\theta(x + \delta, y) \approx m_\theta(x, y) + \langle \nabla_x m_\theta(x, y), \delta\rangle.
\]
Hence by H\"{o}lder's inequality with $1/p + 1/q = 1$,
\[
\max_{\|\delta\|_p \le \varepsilon} |m_\theta(x + \delta, y) - m_\theta(x, y)|
= \varepsilon\, \|\nabla_x m_\theta(x, y)\|_q,
\]
so robustness at $x$ is approximate to $m_\theta(x, y) > \varepsilon\, \|\nabla_x m_\theta(x, y)\|_q$.
\end{proposition}
\par Appendix~\ref{sec:remark} provides the remark of Proposition~\ref{prop:dual_gradients}.
\begin{proposition}[Bounding the input Jacobian with layer spectral norms]
\label{prop:bound_jacobian}
For a feedforward network $f(x)=W_{L} \phi\left(\cdots \phi\left(W_{1} x\right)\right)$ with $L$ layers, if the activations $\phi$ are 1-Lipschitz (e.g., ReLU, leaky-ReLU), then
\begin{equation}
\label{eq:jacobian}
\left\|J_{f}(x)\right\|_{2} \leq \prod_{\ell=1}^{L}\left\|W_{\ell}\right\|_{2}.
\end{equation}
This gives $\left\|\nabla_{x} g_\theta(x)\right\|_{2} \leq C \cdot \prod_{\ell}\left\|W_{\ell}\right\|_{2}$, where $C$ absorbs activation Lipschitz constants and output projections.
\end{proposition}
\par The Muon optimizer uses the Newton–Schulz approximation for polar decomposition~\cite{jordan2024muon}, whose update rule is defined in Eq.~(\ref{eq:muon-update}). The polar factor is an approximately orthogonal matrix with a spectral norm of 1. The following lemma extends prior analysis~\cite{jordan2024muon,higham2008functions}.
\begin{lemma}
\label{lemma:polar_factor}
Let $A\in\mathbb{R}^{m\times n}$ be full rank with compact SVD
$A=P\Sigma V^\top$, and define $\operatorname{Ortho}(A)=PV^\top$.
Then $\|\operatorname{Ortho}(A)\|_2=1$. Moreover, if $m\ge n$,
\begin{equation}
\operatorname{Ortho}(A)^\top\operatorname{Ortho}(A)=I_n,\qquad
A=\operatorname{Ortho}(A)(A^\top A)^{1/2};
\end{equation}
if $m<n$, we have $\operatorname{Ortho}(A)\operatorname{Ortho}(A)^\top=I_m,\
A=(AA^\top)^{1/2}\operatorname{Ortho}(A)$. Equivalently, $\operatorname{Ortho}(A)$ is the Frobenius projection of $A$
onto the corresponding rectangular Stiefel set.
\end{lemma}
\par The proof of Lemma~\ref{lemma:polar_factor} is given in Appendix~\ref{subsec:proof:lemma:polar_factor}.
\begin{theorem}
\label{thm:upper_bound_spectral_norm}
Assume there is a quantifiable orthogonalization error $\varepsilon_{t} \geq 0$, such that 
\begin{equation}
\label{eq:muon_trajectory_assumption}
\left\|\widetilde{U}_{t}\right\|_{2} \leq 1+\varepsilon_{t}.
\end{equation}
The ideal Muon optimizer (exact polar) corresponds to $\varepsilon_{t}=0$, while the finite-step Newton--Schulz approximate polar factor corresponds to a small value $\varepsilon_{t}>0$.
\par Under the assumption of Eq.~(\ref{eq:muon_trajectory_assumption}) and $0\le \eta\lambda\le 1$, for any $t \geq 0$, the spectral norm satisfies the following recursive upper bound
\begin{equation}
\label{eq:upper_bound1}
\left\|W_{t+1}\right\|_{2} \leq(1-\eta \lambda)\left\|W_{t}\right\|_{2}+\eta\left(1+\varepsilon_{t}\right).
\end{equation}
Furthermore, if $\eta \lambda \in(0,1)$, the following explicit trajectory upper bound holds
\begin{equation}
\label{eq:trajectory_upper_bound}
\begin{aligned}
\left\|W_{t}\right\|_{2}  \leq(1-\eta \lambda)^{t}\left\|W_{0}\right\|_{2} +\frac{1-(1-\eta \lambda)^{t}}{\lambda} +\eta \sum_{k=0}^{t-1}(1-\eta \lambda)^{t-1-k} \varepsilon_{k}.
\end{aligned}
\end{equation}
Specifically, if $\varepsilon_{t} \leq \bar{\varepsilon}$ for all $t$, then, it satisfies
\begin{equation}
\label{eq:eq:trajectory_inequality}
\begin{aligned}
\left\|W_{t}\right\|_{2}  \leq(1-\eta \lambda)^{t}\left\|W_{0}\right\|_{2} +\frac{1+\bar{\varepsilon}}{\lambda}\left(1-(1-\eta \lambda)^{t}\right)  \leq \max \left\{\left\|W_{0}\right\|_{2}, \frac{1+\bar{\varepsilon}}{\lambda}\right\}.
\end{aligned}
\end{equation}
\end{theorem}
\par Appendix~\ref{sec:remark} provides the remark of Theorem~\ref{thm:upper_bound_spectral_norm}. The proof of Theorem~\ref{thm:upper_bound_spectral_norm} is given in Appendix~\ref{subsec:proof:thm:upper_bound_spectral_norm}.
\begin{theorem}
\label{thm:union_robustness}
Fix $\mathcal{P}\subseteq\{1,2,\infty\}$ with radii
$\{\varepsilon_p\}_{p\in\mathcal{P}}$ and let
$\Delta_{\mathrm{union}}=\bigcup_{p\in\mathcal{P}}
\{\delta\in\mathbb{R}^d:\|\delta\|_p\le\varepsilon_p\}$.
Suppose every layer satisfies $\|W_\ell\|_2\le B_\ell$
(as guaranteed by Theorem~\ref{thm:upper_bound_spectral_norm} with weight decay $\lambda>0$),
and define
\begin{equation}
\label{eq:union_robustness_definition}
\hat{L}_p(\theta):=c_p C_{\rm marg}\prod_{\ell=1}^L B_\ell,
\quad
C_{\rm marg}:=\max_{j\ne y}\|e_y-e_j\|_2=\sqrt{2}.
\end{equation}
where $c_p$ is the norm-equivalence constant satisfying $\|\delta\|_{2} \leq c_{p}\|\delta\|_{p}$ for all $\delta \in \mathbb{R}^{d}$
(concretely, $c_{\infty}=\sqrt{d},\ c_{2}=1, \text { and } c_{1}=1$, since $\|\delta\|_{2} \leq \sqrt{d}\|\delta\|_{\infty}$ and $\|\delta\|_{2} \leq \|\delta\|_{1}$). Furthermore, $C_{\rm marg}$ absorbs the output-layer projection from logits to margin.
If
\begin{equation}
  m_\theta(x,y)
  \;>\;\max_{p\in\mathcal{P}}\,\hat{L}_p(\theta)\,\varepsilon_p,
\end{equation}
then $m_\theta(x+\delta,y)>0$ for all $\delta\in\Delta_{\mathrm{union}}$.
\end{theorem}
\par Appendix~\ref{sec:remark} provides the remark of Theorem~\ref{thm:union_robustness}. The proof of Theorem~\ref{thm:union_robustness} is given in Appendix~\ref{subsec:proof:thm:union_robustness}.
\subsection{Muon's Gradient Dynamics under Min-Max Optimization}
\par Our main result (Theorem~\ref{thm:adv_loss_refined}) shows that Muon descends the adversarial loss along the nuclear-norm direction with a structurally larger per-step guarantee than SGD.
\begin{assumption}
\label{assumption:adv_loss}
We can view the adversarial loss $L_{\mathrm{adv}}$ as a function $F(W)$ of a particular weight block $W$. Assume that $F$ is $\beta$-smooth, i.e., the following inequality holds
\begin{equation}
\label{eq:inequality_fw}
F(W+\Delta) \leq F(W)+\langle\nabla F(W), \Delta\rangle+\frac{\beta}{2}\|\Delta\|_{F}^{2},
\end{equation}
where $\Delta$ represents any perturbation or increment to the parameter block $W$.
\end{assumption}
\begin{theorem}[Muon optimizer on the adversarial loss for the min-max dynamics]
\label{thm:adv_loss_refined}
In AT, assume $F$ is $\beta$-smooth w.r.t.\ $\|\cdot\|_F$.
Let $G=\nabla_W F(W)$ be the block gradient and let $M$ be the
momentum-accumulated gradient used by Muon.
Let $\widetilde U$ be the inexact polar direction used in practice, satisfying
\begin{equation}
\label{eq:ortho_eq}
\|\widetilde U-\operatorname{Ortho}(M)\|_F \le \delta_\mathrm{orth},
\qquad
\|\widetilde U\|_2 \le 1+\varepsilon,
\qquad
\|\widetilde U\|_F^2 \le r(1+\varepsilon)^2,
\end{equation}
where $r$ is $\mathrm{rank}(M)$ or any upper bound on it.
Consider the update $W^+=W-\eta\,\widetilde U$ with $\eta>0$.
Then the following inequality holds:
\begin{equation}
\label{eq:muon_descent_refined}
F(W^+) \le F(W)
-\eta\|M\|_*
+\eta(1+\varepsilon)\|G-M\|_*
+\eta\sqrt r\,\|M\|_2\,\delta_\mathrm{orth}
+\frac{\beta}{2}\eta^2\,r(1+\varepsilon)^2.
\end{equation}
\end{theorem}
\par Appendix~\ref{sec:remark} provides the remark of Theorem~\ref{thm:adv_loss_refined}. The proof of Theorem~\ref{thm:adv_loss_refined} is given in Appendix~\ref{subsec:proof:thm:adv_loss_refined}.

\section{Empirical Results and Insights}
\subsection{Adapting Muon for AT}
\par We adopt an AT method derived from a previous seminal work~\cite{croce2022adversarial}. We instantiate the AT objective defined in Eq. \eqref{eq:at-objective} with APGD as the inner attack and Muon as the outer optimizer.
\par The attack method used in our study is Auto Projected Gradient Descent (APGD)~\cite{croce2020reliable}. It replaces the fixed step size of PGD with an adaptive schedule and an automatic restart mechanism. For $\ell_\infty$ and $\ell_2$ attacks, it uses $\mathrm{sign}(\nabla_x \ell)$ and $\nabla_x \ell / \|\nabla_x\ell\|_2$, respectively. For the $\ell_{1}$ attack, the step ascends along the top-k coordinates of $\lvert \nabla_{x}\ell \rvert$ (a sparse steepest-ascent direction) and projects the iterate back onto the $\ell_{1}$ ball of radius $\varepsilon_{1}$ after each step~\cite{croce2022adversarial}.
\par Algorithm~\ref{alg:robust_training} outlines the AT procedure. We sample a norm per mini-batch (e.g., $\ell_\infty$ or $\ell_{1}$) to craft adversarial examples, then update the model parameters. Various optimizers can be integrated as plugins within this AT framework.
\begin{algorithm}[!t]
\caption{Multi-Norm APGD Adversarial Training}
\label{alg:robust_training}
\begin{algorithmic}[1]
\REQUIRE Model $f_\theta$, data loader $\mathcal{D}$,
  threat norms $\mathcal{P}=\{p_1,\dots,p_M\}$,
  radii $\{\varepsilon_i\}_{i=1}^M$, attack steps $\{T_i\}_{i=1}^M$,
  optimizer $\mathrm{Opt}$ (SGD / AdamW / Muon),
  adversarial loss $\ell(\cdot,\cdot)$,
  optional scheduler $\mathrm{Scheduler}$
\STATE Initialize parameters $\theta$ (random or pretrained)
\FOR{epoch $=1$ to $K$}
  \FOR{each minibatch $(x,y)\sim\mathcal{D}$}
    \STATE $i\gets\mathrm{SampleThreatIndex}(\{1,\dots,M\})$
    \STATE $p\gets p_i,\ \varepsilon\gets\varepsilon_i,\ T\gets T_i$
    \STATE $x_{\mathrm{adv}}\gets\mathrm{APGD}(f_\theta,x,y;\,p,\varepsilon,T,\ell)$
    \STATE $\mathrm{Opt.zero\_grad}()$
    \STATE $L_{\mathrm{adv}}\gets\ell(f_\theta(x_{\mathrm{adv}}),y)$
    \STATE $\mathrm{Backprop}(L_{\mathrm{adv}})$
    \STATE $\mathrm{Opt.step}()$
      \hfill{\footnotesize%
        $\triangleright$ \textit{If Muon: see Alg. \ref{alg:muon-step} for the expanded step}}
    \IF{$\mathrm{Scheduler}$ is used}
      \STATE $\mathrm{Scheduler.step}()$
    \ENDIF
  \ENDFOR
\ENDFOR
\end{algorithmic}
\end{algorithm}

\subsection{Experiment Setup}
\noindent\textbf{Datasets}. The CIFAR-10 dataset~\cite{krizhevsky2009learning} is a commonly used dataset in computer vision. It consists of 60,000 $32 \times 32$ color images in 10 classes with 50,000 training images and 10,000 test images. It is a standard benchmark in adversarial robustness research. ImageNet~\cite{deng2009imagenet} is a large-scale image database for visual recognition research. It contains over 14 million human-annotated images spanning more than 20,000 categories, and around 1 million images include object bounding boxes. In research practice, the dataset is often used in its 1,000-class subset, which includes 1,281,167 training samples and 50,000 validation samples.

\noindent\textbf{Training Details}. On CIFAR-10~\cite{krizhevsky2009learning}, the training process lasts for 100 epochs. The APGD method~\cite{croce2020reliable} is used as the default attack method with 10 iteration steps. For the $\ell_\infty$-attack, the perturbation budget is $\varepsilon_{\infty}=\frac{8.0}{255.0}$. For the $\ell_2$-attack, the perturbation budget is $\varepsilon_{2}=0.5$. For the $\ell_1$-attack, the perturbation budget is $\varepsilon_{1}=12$. The batch size is set to 128. For CNNs~\cite{he2016identity, zagoruyko2016wide}, the initial learning rate is 1.0. For ViTs~\cite{dosovitskiy2021image}, the initial learning rate is 0.001. The learning rate is set in a stepwise manner, decreasing to 0.1 and 0.01 of the initial learning rate at 1/3 and 2/3 of the entire training period, respectively. The momentum  coefficient $\mu$ is 0.9, and the weight decay $\lambda$ is 0.0005. On ImageNet~\cite{deng2009imagenet}, the training process lasts for 100 epochs. Due to computing resource constraints, the defense is only deployed on ResNet-50~\cite{he2016deep}. The AT mode is set to the $\ell_{\infty} \ + \ \ell_{1}$ joint norm. The batch size is still set to 128. For the $\ell_\infty$-attack, the perturbation budget is $\varepsilon_{\infty}=\frac{4.0}{255.0}$, and the attack step number is 5. For the $\ell_2$-attack, the perturbation budget is $\varepsilon_{2}=2$, and the attack step number is 10. For the $\ell_1$-attack, the perturbation budget is $\varepsilon_{1}=255$, and the attack step number is 15. 

\subsection{Main Evaluation Results on CIFAR-10}
The models selected for our experiment include PreActResNet-18~\cite{he2016identity}, WRN-34-10~\cite{zagoruyko2016wide}, WRN-34-20~\cite{zagoruyko2016wide}, ViT-B~\cite{dosovitskiy2021image}, and ViT-L~\cite{dosovitskiy2021image}. The selection of best weights is based on their $\ell_{\infty}$ robustness. In practice, the best checkpoint is selected using a small validation subset during training, and the final reported performance is evaluated on the full test set. This protocol may introduce checkpoint-selection bias. We compare Muon~\cite{jordan2024muon}, AdamW~\cite{loshchilov2017decoupled}, and SGD~\cite{robbins1951stochastic,bottou2012stochastic} on NVIDIA RTX 4080S.
\par \textbf{PreActResNet-18}. Table~\ref{tab:multi_norm_at} shows the results on PreActResNet-18 (11.2 million parameters)~\cite{he2016identity}. Under the multi-norm training ($\ell_{\infty} \ + \ \ell_{1}$), Muon's best checkpoint achieves the union robustness of 39.27, compared to SGD's 40.37 and AdamW's 36.53. Its robustness is optimal only under the $\ell_{2}$-norm training setting. On shallower CNNs, Muon is not necessarily superior to SGD. Muon's advantage mainly lies in its balanced performance across norms.
\begin{table}[!t]
\centering
\footnotesize
\caption{AT results on PreActResNet-18~\cite{he2016identity}.}
\label{tab:multi_norm_at}
\setlength{\tabcolsep}{5pt}
\renewcommand{\arraystretch}{1.05}
\resizebox{\columnwidth}{!}{%
\begin{tabular}{llccccc@{\hspace{6pt}}ccccc}
\toprule
\multirow{2}{*}{Norm} & \multirow{2}{*}{Opt.} 
& \multicolumn{5}{c}{Best} 
& \multicolumn{5}{c}{Last} \\
\cmidrule(lr){3-7}\cmidrule(l){8-12}
& 
& Clean & $\ell_\infty$ & $\ell_2$ & $\ell_1$ & Union
& Clean & $\ell_\infty$ & $\ell_2$ & $\ell_1$ & Union \\
\midrule

\multirow{3}{*}{$\ell_\infty{+}\ell_1$}
& SGD   & \textbf{79.21} & \textbf{40.84} & \textbf{64.86} & \textbf{52.15} & \textbf{40.37} & 81.22 & 35.77 & 64.16 & 48.00 & 35.07 \\
& AdamW & 66.26 & 36.83 & 55.19 & 45.59 & 36.53 & 69.98 & 32.78 & 55.36 & 43.76 & 32.38 \\
& Muon  & 77.48 & 39.67 & 63.08 & 51.03 & 39.27 & \textbf{84.81} & \textbf{38.58} & \textbf{67.24} & \textbf{50.90} & \textbf{37.83} \\
\midrule

\multirow{3}{*}{$\ell_\infty$}
& SGD   & \textbf{78.60} & \textbf{45.81} & 58.20 & 8.05 & 8.05 & 82.37 & 41.73 & 57.17 & 5.80 & 5.79 \\
& AdamW & 65.18 & 39.64 & 47.76 & 10.07 & 10.07 & 69.22 & 38.63 & 48.96 & \textbf{7.87} & \textbf{7.87} \\
& Muon  & 78.04 & 43.86 & \textbf{58.58} & \textbf{10.16} & \textbf{10.16} & \textbf{85.74} & \textbf{43.68} & \textbf{60.84} & 6.00 & 6.00 \\
\midrule

\multirow{3}{*}{$\ell_1$}
& SGD   & \textbf{84.08} & 17.10 & \textbf{59.15} & \textbf{53.05} & 17.10 & 84.08 & 16.98 & 59.34 & 53.20 & 16.98 \\
& AdamW & 71.50 & 21.83 & 54.64 & 49.62 & 21.83 & 79.36 & \textbf{19.09} & 54.36 & 49.68 & \textbf{19.09} \\
& Muon  & 76.30 & \textbf{23.63} & 57.75 & 51.71 & \textbf{23.63} & \textbf{86.68} & 18.19 & \textbf{63.50} & \textbf{57.57} & 18.19 \\
\midrule

\multirow{3}{*}{$\ell_2$}
& SGD   & 87.58 & 25.91 & \textbf{65.53} & 25.00 & 20.52 & 87.74 & 22.81 & 63.46 & 22.73 & 17.82 \\
& AdamW & 80.74 & 24.58 & 58.84 & 25.42 & 19.97 & 82.48 & \textbf{23.31} & 58.92 & 24.11 & 18.66 \\
& Muon  & \textbf{85.10}	 & \textbf{28.52} & 64.09 & \textbf{27.75} & \textbf{22.85} & \textbf{89.66} & 23.82 & \textbf{66.94} & \textbf{25.75} & \textbf{19.63} \\
\bottomrule
\end{tabular}}
\end{table}
\par \textbf{WideResNet}. Table~\ref{tab:wrn34_10_multi_norm_at} illustrates the robustness evaluation results on WRN-34-10 (48.2 million parameters)~\cite{zagoruyko2016wide}. Muon achieves the strongest union robustness under AT. With $\ell_{\infty}$-only training, Muon attains a high $\ell_{\infty}$ robust accuracy (46.71), but its $\ell_{1}$ robustness nearly collapses (4.88). It reflects a general phenomenon where training against a single threat model tends to produce robustness mainly against that specific threat. Table~\ref{tab:wrn34_20_multi_norm_at} further shows robustness evaluation results on WRN-34-20 (192.9 million parameters)~\cite{zagoruyko2016wide}. Muon’s best checkpoint achieves a union robustness of 39.15. Furthermore, the Muon optimizer attains higher $\ell_{\infty}/\ell_{2}$ robustness (39.99/61.22). Under both  $\ell_{\infty} \ + \ \ell_{1}$ and $\ell_\infty$ AT with the Muon optimizer, the $\ell_{1}$ robust accuracy of the final checkpoint drops sharply, indicating that Muon needs an early-stopping strategy to mitigate ``robust overfitting''.
\begin{table}[!t]
\centering
\footnotesize
\caption{AT results on WRN-34-10~\cite{zagoruyko2016wide}.}
\label{tab:wrn34_10_multi_norm_at}
\setlength{\tabcolsep}{5pt}
\renewcommand{\arraystretch}{1.05}
\resizebox{\columnwidth}{!}{%
\begin{tabular}{llccccc@{\hspace{8pt}}ccccc}
\toprule
\multirow{2}{*}{Norm} & \multirow{2}{*}{Optimizer}
& \multicolumn{5}{c}{Best}
& \multicolumn{5}{c}{Last} \\
\cmidrule(lr){3-7}\cmidrule(l){8-12}
&
& Clean & $\ell_\infty$ & $\ell_2$ & $\ell_1$ & Union
& Clean & $\ell_\infty$ & $\ell_2$ & $\ell_1$ & Union \\
\midrule

\multirow{3}{*}{$\ell_1+\ell_\infty$}
& SGD   & \textbf{78.19} & 38.18 & \textbf{63.29} & \textbf{50.51} & 37.74 & 81.52 & 32.03 & 61.43 & 43.71 & 31.36 \\
& AdamW & 75.43 & 37.80 & 60.42 & 47.49 & 37.15 & 78.90 & 35.72 & 61.48 & 46.57 & 34.99 \\
& Muon  & 73.27 & \textbf{38.87} & 58.44 & 45.73 & \textbf{37.89} & \textbf{85.21} & \textbf{38.14} & \textbf{65.13} & \textbf{46.58} & \textbf{37.09} \\
\midrule

\multirow{3}{*}{$\ell_\infty$}
& SGD   & 80.02 & 44.11 & 57.73 & \textbf{7.95} & \textbf{7.95} & 82.50 & 37.23 & 53.21 & \textbf{5.69} & \textbf{5.69} \\
& AdamW & 69.92 & 42.47 & 49.69 & 4.41 & 4.41 & 72.39 & 41.38 & 48.25 & 3.18 & 3.18 \\
& Muon  & \textbf{83.16} & \textbf{46.71} & \textbf{58.11} & 4.88 & 4.88 & \textbf{85.96} & \textbf{43.81} & \textbf{54.05} & 4.34 & 4.34 \\
\midrule

\multirow{3}{*}{$\ell_1$}
& SGD   & 47.59 & 19.39 & 36.99 & 33.27 & 19.39 & 83.90 & 16.80 & 58.03 & 51.04 & 16.80 \\
& AdamW & 65.26 & 23.49 & 50.97 & 46.03 & 23.49 & 75.77 & \textbf{20.91} & 55.29 & 49.41 & \textbf{20.91}\\
& Muon  & \textbf{78.81} & \textbf{24.89} & \textbf{58.87} & \textbf{53.56} & \textbf{24.89} & \textbf{87.83} & 19.60 & \textbf{63.03} & \textbf{58.15} & 19.60 \\
\midrule

\multirow{3}{*}{$\ell_2$}
& SGD   & 76.59& 22.72 & 54.04 & 23.72 & 18.87 & 87.51 & 21.49 & 62.31 & 24.65 & 18.18 \\
& AdamW & 73.40 & 24.07 & 51.93 & 24.18 & 19.90 & 86.17 & 21.65 & 61.15 & 19.38 & 15.75 \\
& Muon  & \textbf{89.97} & \textbf{26.85} & \textbf{68.16} & \textbf{25.96} & \textbf{21.63} & \textbf{90.18} & \textbf{26.44} & \textbf{68.56} & \textbf{26.59} & \textbf{21.53} \\
\bottomrule
\end{tabular}%
\vspace{-10pt}
}
\end{table}
\begin{table}[!t]
\centering
\footnotesize
\caption{AT results on WRN-34-20~\cite{zagoruyko2016wide}.}
\label{tab:wrn34_20_multi_norm_at}
\setlength{\tabcolsep}{5pt}
\renewcommand{\arraystretch}{1.05}
\resizebox{\columnwidth}{!}{%
\begin{tabular}{llccccc@{\hspace{8pt}}ccccc}
\toprule
\multirow{2}{*}{Norm} & \multirow{2}{*}{Optimizer}
& \multicolumn{5}{c}{Best}
& \multicolumn{5}{c}{Last} \\
\cmidrule(lr){3-7}\cmidrule(l){8-12}
&
& Clean & $\ell_\infty$ & $\ell_2$ & $\ell_1$ & Union
& Clean & $\ell_\infty$ & $\ell_2$ & $\ell_1$ & Union \\
\midrule

\multirow{3}{*}{$\ell_1+\ell_\infty$}
& SGD   & 75.35 & 36.84 & 59.99 & \textbf{48.18} & 36.51 & 82.58 & 34.14 & \textbf{62.16} & \textbf{44.79} & 33.26 \\
& AdamW & 72.39 & 35.44 & 57.51 & 43.99 & 34.65 & 72.39 & 35.43 & 57.51 & 44.04 & \textbf{34.63} \\
& Muon  & \textbf{76.41} & \textbf{39.99} & \textbf{61.22} & 47.68 & \textbf{39.15} & \textbf{86.51} & \textbf{37.87} & 61.69 &  11.10 &  11.10 \\
\midrule

\multirow{3}{*}{$\ell_\infty$}
& SGD   & 75.98 & 41.12 & 55.09 & \textbf{8.30} & \textbf{8.30} & 83.62 & 39.17 & 54.40 &  \textbf{5.34} &  \textbf{5.34} \\
& AdamW & 73.37 & 42.59 & 48.31 &  2.48 & 2.48 & 75.17 & 40.91 & 47.35 &  2.22 &  2.22 \\
& Muon  & \textbf{82.13} & \textbf{46.58} & \textbf{57.98} &  5.30 &  5.30 & \textbf{87.44} & \textbf{46.76} & \textbf{56.99} & 5.15 & 5.15 \\
\midrule

\multirow{3}{*}{$\ell_1$}
& SGD   & 56.49 & 21.24 & 41.95 & 37.74 & 21.24 & 85.31 & 17.94 & 60.45 & 54.45 & 17.94 \\
& AdamW & 68.56 & 21.79 & 52.11 & 47.44 & 21.79 & 82.03 & 17.37 & 57.17 & 50.30 & 17.37 \\
& Muon  & \textbf{88.53} & \textbf{21.45} & \textbf{65.71} & \textbf{60.19} & \textbf{21.45} & \textbf{88.61} & \textbf{20.83} & \textbf{65.16} & \textbf{59.95} & \textbf{20.83} \\
\midrule

\multirow{3}{*}{$\ell_2$}
& SGD   & 88.33 & 24.37 & 64.29 & 25.26 & 20.11 & 88.48 & 23.93 & 64.84 & 25.66 & 20.05 \\
& AdamW & 85.55 & 22.73 & 61.20 & 22.15 & 17.92 & 85.45 & 22.59 & 60.76 & 21.53 & 17.76 \\
& Muon  & \textbf{91.09} & \textbf{28.15} & \textbf{70.14} & \textbf{27.26} & \textbf{22.64} & \textbf{91.16} & \textbf{27.84} & \textbf{70.38} & \textbf{27.25} & \textbf{22.44} \\
\bottomrule
\end{tabular}%
\vspace{-10pt}
}
\end{table}
\par \textbf{Vision Transformers}. Table~\ref{tab:vit_b_multi_norm_at} shows the robustness of ViT-B (85.7 million parameters)~\cite{dosovitskiy2021image} under different optimizers. Muon is clearly beneficial for AT on ViT-B. Compared with AdamW, Muon substantially improves the stability of AT for Transformers. While SGD can also achieve a certain level of robustness on ViT-B, Muon’s advantage is that it remains relatively stable across AT scenarios. Single-norm training still specializes to its own threat for all optimizers. The distinction is that Muon avoids the training instability that plagues AdamW on Transformers. Table~\ref{tab:vit_l_multi_norm_at} shows the robustness of ViT-L (303.1 million parameters)~\cite{dosovitskiy2021image} under different optimizers. The experimental observations on ViT-L are similar to those on ViT-B. On ViT-L, AdamW exhibits partial training failure, analogous to what happens when training ViT-B under the $\ell_{\infty}$-norm. To stabilize AT with AdamW, we set the learning rate to 0.00005 and the weight decay to 0.02. We find that AdamW’s final model under $\ell_\infty$-only training on ViT-L is close to random performance (e.g., the union robust accuracy is low), whereas Muon still maintains considerable robustness. However, AT with Muon can also suffer from ``robust overfitting''. 
\begin{table}[!t]
\centering
\footnotesize
\caption{AT results on ViT-B~\cite{dosovitskiy2021image}.}
\label{tab:vit_b_multi_norm_at}
\setlength{\tabcolsep}{5pt}
\renewcommand{\arraystretch}{1.05}
\resizebox{\columnwidth}{!}{%
\begin{tabular}{llccccc@{\hspace{8pt}}ccccc}
\toprule
\multirow{2}{*}{Norm} & \multirow{2}{*}{Optimizer}
& \multicolumn{5}{c}{Best}
& \multicolumn{5}{c}{Last} \\
\cmidrule(lr){3-7}\cmidrule(l){8-12}
&
& Clean & $\ell_\infty$ & $\ell_2$ & $\ell_1$ & Union
& Clean & $\ell_\infty$ & $\ell_2$ & $\ell_1$ & Union \\
\midrule

\multirow{3}{*}{$\ell_1+\ell_\infty$}
& SGD   & \textbf{66.35} & \textbf{29.64} & \textbf{52.69} & \textbf{44.78} & \textbf{29.62} & \textbf{71.96} & \textbf{24.88} & \textbf{52.94} & \textbf{40.32} & \textbf{24.80} \\
& AdamW & 27.22 & 12.10 & 18.97 & 16.54 & 12.02 & 32.77 & 13.62 & 23.11 & 19.88 & 13.53 \\
& Muon  & 61.53 & 27.03 & 48.63 & 42.61 & 27.03 & 65.49 & 19.17 & 45.79 & 35.42 & 19.15 \\
\midrule

\multirow{3}{*}{$\ell_\infty$}
& SGD   & \textbf{64.65} & \textbf{30.97} & \textbf{50.42} & \textbf{32.59} & \textbf{28.10} & \textbf{71.64} & \textbf{26.15} & \textbf{52.58} & \textbf{27.86} & \textbf{27.71} \\
& AdamW & 17.93 & 12.12 & 11.60 & 4.48 & 4.48 & 17.25 & 12.93 & 12.31 & 5.77 & 5.77 \\
& Muon  & 60.66 & 30.36 & 47.20 & 28.00 & 25.73 & 66.26 & 19.93 & 44.20 & 26.23 & 19.30 \\
\midrule

\multirow{3}{*}{$\ell_1$}
& SGD   & \textbf{68.48} & \textbf{22.63} & \textbf{51.30} & \textbf{45.27} & \textbf{22.63} & \textbf{72.24} & \textbf{17.19} & \textbf{49.72} & \textbf{43.06} & \textbf{17.19} \\
& AdamW & 25.57 & 8.99 & 17.08 & 15.94 & 8.97 & 28.26 & 6.47 & 16.98 & 15.66 & 6.44 \\
& Muon  & 58.58 & 21.15 & 44.41 & 40.93 & 21.15 & 64.64 &  12.15 & 40.61 & 34.79 &  12.15 \\
\midrule

\multirow{3}{*}{$\ell_2$}
& SGD   & 57.08 & 17.49 & 40.32 & 29.44 & 17.48 & \textbf{76.50} & \textbf{18.61} & \textbf{52.99} & \textbf{33.55} & \textbf{18.46} \\
& AdamW & 18.51 & 11.53 & 13.95 & 13.01 & 11.49 & 35.10 & 8.81 & 20.54 & 16.13 & 8.71 \\
& Muon  & \textbf{58.49} & \textbf{17.66} & \textbf{41.01} & \textbf{31.77} & \textbf{17.62} & 67.89 &  12.70 & 42.60 & 30.96 &  12.70 \\
\bottomrule
\end{tabular}%
}
\end{table}
\begin{table}[!t]
\centering
\footnotesize
\caption{AT results on ViT-L~\cite{dosovitskiy2021image}.}
\label{tab:vit_l_multi_norm_at}
\setlength{\tabcolsep}{5pt}
\renewcommand{\arraystretch}{1.05}
\resizebox{\columnwidth}{!}{%
\begin{tabular}{llccccc@{\hspace{8pt}}ccccc}
\toprule
\multirow{2}{*}{Norm} & \multirow{2}{*}{Optimizer}
& \multicolumn{5}{c}{Best}
& \multicolumn{5}{c}{Last} \\
\cmidrule(lr){3-7}\cmidrule(l){8-12}
&
& Clean & $\ell_\infty$ & $\ell_2$ & $\ell_1$ & Union
& Clean & $\ell_\infty$ & $\ell_2$ & $\ell_1$ & Union \\
\midrule

\multirow{3}{*}{$\ell_1+\ell_\infty$}
& SGD   & \textbf{69.87} & \textbf{32.51} & \textbf{55.86} & \textbf{47.39} & \textbf{32.50} & \textbf{73.10} & \textbf{26.69} & \textbf{55.09} & \textbf{43.03} & \textbf{26.64} \\
& AdamW & 26.52 & 11.43 & 19.14 & 16.28 & 11.23 & 25.10 & 10.62 & 17.12 & 15.19 & 10.59 \\
& Muon  & 63.85 & 28.18 & 50.85 & 43.94 & 28.18 & 68.02 & 21.65 & 47.79 & 37.25 & 21.64 \\
\midrule

\multirow{3}{*}{$\ell_\infty$}
& SGD   & \textbf{70.15} & \textbf{34.26} & \textbf{55.30} & \textbf{34.95} & \textbf{30.25} & \textbf{72.32} & \textbf{25.61} & \textbf{51.95} & \textbf{27.67} & \textbf{22.68} \\
& AdamW & 16.98 & 12.85 & 14.12 & 10.41 & 10.32 & 18.05 & 6.13 & 8.14 &  0.79 &  0.65 \\
& Muon  & 60.42 & 31.79 & 47.76 & 28.72 & 26.61 & 68.35 & 23.03 & 48.16 & 27.52 & 21.95 \\
\midrule

\multirow{3}{*}{$\ell_1$}
& SGD   & 24.68 & 9.22 & 18.22 & 18.58 & 9.22 & 40.60 & 13.48 & 29.27 & 28.23 & 13.48 \\
& AdamW & 17.84 & 11.12 & 14.90 & 14.83 & 11.12 & 24.27 &  5.62 & 13.27 & 13.02 &  5.62 \\
& Muon  & \textbf{58.00} & \textbf{21.13} & \textbf{43.75} & \textbf{39.80} & \textbf{21.13} & \textbf{67.47} & \textbf{14.05} & \textbf{44.00} & \textbf{39.09} & \textbf{14.05} \\
\midrule

\multirow{3}{*}{$\ell_2$}
& SGD   & \textbf{77.03} & 18.37 & 55.32 & 33.31 & 18.15 & \textbf{77.22} & \textbf{19.01} & \textbf{54.40} & \textbf{34.42} & \textbf{18.78} \\
& AdamW & 21.31 & 10.24 & 16.63 & 14.63 & 10.24 & 26.38 & 5.77 & 14.80 & 11.62 & 5.73 \\
& Muon  & 67.40 & \textbf{21.40} & \textbf{48.68} & \textbf{35.54} & \textbf{21.33} & 71.17 & 15.32 & 46.74 & 31.90 & 15.26 \\
\bottomrule
\end{tabular}%
}
\end{table}
\par In summary, AT with the Muon optimizer achieves considerable robustness on overparameterized CNNs. Muon achieves union robustness on par with or better than SGD on WRN-34-10/20, while outperforming AdamW by a large margin. On ViT-B and ViT-L, both Muon and SGD yield stable AT, while AdamW may collapse to near-random performance. Muon offers the most consistent $\ell_\infty$ robustness among the three, whereas SGD remains competitive on union robustness. 
\subsection{Discussion}
\noindent\textbf{Robust Overfitting}. We observe the symptom of ``robust overfitting'' in the experiments. It indicates that Muon improves
optimization stability in some cases but does not eliminate the late-stage robustness
degradation. It suggests that the ``early stopping" strategy remains
important for Muon-based AT as well as for SGD-based AT. We discuss the experimental phenomenon in Appendix~\ref{sec:robust_overfitting}.

\noindent\textbf{Training Time}. Table~\ref{tab:train_time_wrn34_20} reports the AT time across different optimizers. We train WRN-34-20 for 100 epochs. The training time is comparable across the three optimizers.
\begin{table}[!t]
\centering
\footnotesize
\caption{Training time for WRN-34-20~\cite{zagoruyko2016wide} under different threat norms. Times are converted from seconds to hours (h) and minutes (m), rounded to the nearest minute.}
\setlength{\tabcolsep}{3pt}
\renewcommand{\arraystretch}{1.05}
\resizebox{\columnwidth}{!}{%
\begin{tabular}{lcccccccccccc}
\toprule
& \multicolumn{4}{c}{SGD} & \multicolumn{4}{c}{AdamW} & \multicolumn{4}{c}{Muon} \\
\cmidrule(lr){2-5}\cmidrule(lr){6-9}\cmidrule(l){10-13}
Metric
& $\ell_1{+}\ell_\infty$ & $\ell_\infty$ & $\ell_1$ & $\ell_2$
& $\ell_1{+}\ell_\infty$ & $\ell_\infty$ & $\ell_1$ & $\ell_2$
& $\ell_1{+}\ell_\infty$ & $\ell_\infty$ & $\ell_1$ & $\ell_2$ \\
\midrule
Time
& 51h59m & 51h47m & 52h18m & 51h49m
& 52h06m & 51h52m & 52h21m & 51h56m
& 52h45m & 52h29m & 53h04m & 52h32m \\
\bottomrule
\end{tabular}%
}
\label{tab:train_time_wrn34_20}
\end{table}

\noindent\textbf{Comparison with SAM}. We conduct a comparison experiment between Muon and SAM ($\rho=0.05$) on WRN-34-10 (WRN-34-10)~\cite{zagoruyko2016wide} and ViT-B~\cite{dosovitskiy2021image}. Table~\ref{tab:wrn3410_vitb_sam_muon} demonstrates the robustness advantage of the SAM optimizer. Nevertheless, SAM requires two backpropagations. Taking $\ell_{2}$-norm AT on WRN-34-10 as an example, SAM takes 67,741.1\,s, while Muon requires only 60,415.0\,s.
\begin{table*}[!t]
\centering
\footnotesize
\caption{AT results on WRN-34-10 and ViT-B (SAM vs. Muon).}
\label{tab:wrn3410_vitb_sam_muon}
\setlength{\tabcolsep}{5pt}
\renewcommand{\arraystretch}{1.03}
\begin{tabular}{lllccccc@{\hspace{10pt}}ccccc}
\toprule
\multirow{2}{*}{Model} & \multirow{2}{*}{Norm} & \multirow{2}{*}{Optimizer}
& \multicolumn{5}{c}{Best}
& \multicolumn{5}{c}{Last} \\
\cmidrule(lr){4-8}\cmidrule(l){9-13}
&&
& Clean & $\ell_\infty$ & $\ell_2$ & $\ell_1$ & Union
& Clean & $\ell_\infty$ & $\ell_2$ & $\ell_1$ & Union \\
\midrule

\multirow{8}{*}{WRN-34-10}
& \multirow{2}{*}{$\ell_1+\ell_\infty$}
& SAM  & \textbf{78.75} & \textbf{40.98} & \textbf{64.64} & \textbf{52.77} & \textbf{40.64} & 80.93 & \textbf{40.93} & \textbf{65.84} & \textbf{52.71} & \textbf{40.52} \\
& & Muon & 73.27 & 38.87 & 58.44 & 45.73 & 37.89 & \textbf{85.21} & 38.14 & 65.13 & 46.58 & 37.09 \\

\cmidrule(lr){2-13}
& \multirow{2}{*}{$\ell_\infty$}
& SAM  & 80.53 & 46.08 & \textbf{60.30} & \textbf{10.43} & \textbf{10.43} & 81.84 & \textbf{46.20} & \textbf{60.76} & \textbf{9.74} & \textbf{9.73} \\
& & Muon & \textbf{83.16} & \textbf{46.71} & 58.11 & 4.88 & 4.88 & \textbf{85.96} & 43.81 & 54.06 & 4.34 & 4.34 \\

\cmidrule(lr){2-13}
& \multirow{2}{*}{$\ell_1$}
& SAM  & \textbf{83.93} & 24.87 & \textbf{63.78} & \textbf{57.57} & 24.86 & 85.04 & \textbf{22.34} & \textbf{63.92} & 57.66 & \textbf{22.34} \\
& & Muon & 78.81 & \textbf{24.89} & 58.87 & 53.56 & \textbf{24.89} & \textbf{87.83} & 19.60 & 63.03 & \textbf{58.15} & 19.60 \\

\cmidrule(lr){2-13}
& \multirow{2}{*}{$\ell_2$}
& SAM  & 87.95 & \textbf{29.27} & 67.50 & \textbf{29.34} & \textbf{24.23} & 89.28 & 25.76 & 67.49 & \textbf{26.62} & 21.21 \\
& & Muon & \textbf{89.97} & 26.85 & \textbf{68.16} & 25.96 & 21.63 & \textbf{90.18} & \textbf{26.44} & \textbf{68.56} & 26.59 & \textbf{21.53} \\
\midrule

\multirow{8}{*}{ViT-B}
& \multirow{2}{*}{$\ell_1+\ell_\infty$}
& SAM  & 56.48 & \textbf{28.05} & 45.24 & 40.01 & \textbf{28.04} & 57.83 & \textbf{28.99} & \textbf{46.35} & \textbf{41.19} & \textbf{28.99} \\
& & Muon & \textbf{61.53} & 27.03 & \textbf{48.63} & \textbf{42.61} & 27.03 & \textbf{65.49} & 19.17 & 45.79 & 35.42 & 19.15 \\

\cmidrule(lr){2-13}
& \multirow{2}{*}{$\ell_\infty$}
& SAM  & 51.70 & 28.08 & 40.76 & 24.79 & 23.44 & 51.88 & \textbf{28.52} & 40.99 & 24.83 & \textbf{23.66} \\
& & Muon & \textbf{60.66} & \textbf{30.36} & \textbf{47.20} & \textbf{28.00} & \textbf{25.73} & \textbf{66.26} & 19.93 & \textbf{44.20} & \textbf{26.23} & 19.30 \\

\cmidrule(lr){2-13}
& \multirow{2}{*}{$\ell_1$}
& SAM  & \textbf{66.03} & \textbf{24.65} & \textbf{50.57} & \textbf{46.24} & \textbf{24.65} & \textbf{67.63} & \textbf{25.26} & \textbf{52.11} & \textbf{47.21} & \textbf{25.26} \\
& & Muon & 58.58 & 21.15 & 44.41 & 40.93 & 21.15 & 64.64 & 12.15 & 40.61 & 34.79 & 12.15 \\

\cmidrule(lr){2-13}
& \multirow{2}{*}{$\ell_2$}
& SAM  & \textbf{75.30} & \textbf{24.98} & \textbf{55.97} & \textbf{37.28} & \textbf{24.57} & \textbf{75.45} & \textbf{24.91} & \textbf{56.33} & \textbf{37.17} & \textbf{24.49} \\
& & Muon & 58.49 & 17.66 & 41.01 & 31.77 & 17.62 & 67.89 & 12.70 & 42.60 & 30.96 & 12.70 \\
\bottomrule
\end{tabular}
\end{table*}

\noindent\textbf{Results on ImageNet}. Although AT on ImageNet~\cite{deng2009imagenet} incurs substantial overhead, the exploration is still meaningful. We find that SGD admits a large initial learning rate (e.g., 1.0) on ResNet-50~\cite{he2016deep}, whereas Muon and AdamW require a smaller initial learning rate (0.0001). We test 2,000 samples from the validation set.
Table~\ref{tab:imagenet_results} shows that SGD performs best among the tested optimizers under our current hyperparameters and compute budgets on ImageNet. We do not interpret this result as ruling out Muon-based AT at scale. Rather, it suggests that Muon is more sensitive to learning rate, warmup, and weight decay choices at scale~\cite{liu2025cosmos}. Further ImageNet-scale tuning on Muon-based AT is needed before drawing definitive conclusions.
\begin{table}[!t]
\centering
\footnotesize
\caption{Results on the ImageNet validation set~\cite{deng2009imagenet}.}
\label{tab:imagenet_results}
\setlength{\tabcolsep}{5pt}
\renewcommand{\arraystretch}{1.05}
{%
\begin{tabular}{lccccc@{\hspace{8pt}}ccccc}
\toprule
\multirow{2}{*}{Optimizer}
& \multicolumn{5}{c}{Best}
& \multicolumn{5}{c}{Last} \\
\cmidrule(lr){2-6}\cmidrule(l){7-11}
& Clean & $\ell_\infty$ & $\ell_2$ & $\ell_1$ & Union
& Clean & $\ell_\infty$ & $\ell_2$ & $\ell_1$ & Union \\
\midrule
SGD   & \textbf{56.02} & \textbf{22.22} & \textbf{37.30} & \textbf{23.63} & \textbf{20.17} & \textbf{51.03} &  8.40 & \textbf{23.39} & \textbf{11.72} &  \textbf{7.76} \\
AdamW & 42.43 & 7.62 & 17.87 & 2.88 &  2.49 & 42.09 &  7.28 & 18.41 & 9.18 &  6.35 \\
Muon  & 48.44 & 19.04 & 31.45 & 18.12 & 16.21 & 50.93 &  \textbf{9.33} & 16.80 &  2.98 &  2.50 \\
\bottomrule
\end{tabular}%
}
\end{table}

\section{Conclusion}
\par \textit{The proposal of the Muon optimizer has opened up a new paradigm for robustness research: can this optimizer, which orthogonalizes updates via polar decomposition, achieve effects comparable to gradient descent optimizers in driving robustness?} From a theoretical perspective, we connect Muon's orthogonalized updates to a spectral-norm stability ceiling and a per-step nuclear-norm descent guarantee, characterizing the optimization dynamics induced by min-max training. 
The experimental results show that the Muon optimizer converges faster under AT (see Appendix~\ref{sec:loss_fig}). On CIFAR-10, iMuon is consistently stronger than AdamW and is competitive with SGD across many CNN and ViT settings, but it does not uniformly dominate SGD. Its main benefit is not a universal robustness improvement, but a stable orthogonalized update geometry that often improves early robust optimization. On ImageNet, our pilot study shows that SGD remains stronger under the current tuning budget, indicating that large-scale Muon-based AT requires further hyperparameter and schedule design. 

\bibliography{main}
\bibliographystyle{plain}
\newpage
\appendix
\onecolumn
\section{Related Work}
\par Optimizer design is of significant importance in AT, which relies on the min-max optimization mechanism~\cite{madry2018towards,zhang2019theoretically}. These studies postulate that a stronger inner attack yields a more robust defense model. If the attack is too weak, it can lead to gradient masking or an overestimation of robustness.
\par Most theoretical research on AT focuses on the dynamics of min-max optimization, while less attention is paid to the choice of optimizer itself~\cite{wang2019convergence,gao2019convergence,deng2020towards}. Early studies found that SGD optimization can lead to a more robust loss landscape when combined with regularization and data augmentation~\cite{liu2020bad}. AT with the SGD optimizer and weight decay regularization can achieve feature purification~\cite{allen2022feature}. Empirical studies have found it more suitable to apply SGD with momentum rather than other adaptive optimizers~\cite{pang2021bag}. 
\par SAM minimizes both the loss value and its sharpness~\cite{foret2021sharpness}. It can help improve the robustness of neural networks~\cite{zhang2024duality,zhou2025sharpness}. The implicit bias of SAM yields a better generalization solution than that of standard gradient descent~\cite{andriushchenko2022towards}. At the same time, SAM tends to ``bounce'' on both sides along the direction of maximum curvature; in more general cases, the effect of SAM is linked to the regularization of the Hessian spectrum (especially the maximum eigenvalue)~\cite{bartlett2023dynamics}. However, SAM requires two forward-backward passes per step: one for the gradient perturbation and another for the actual update~\cite{foret2021sharpness}.
\par Previous studies have stated that SGD-trained neural networks provide better robustness to input perturbations than those trained with adaptive gradient methods due to smaller Lipschitz constants~\cite{ma2023understanding}. Recently, the Muon optimizer~\cite{jordan2024muon,liu2025muon} has been proposed in the training process of new foundation models. When the Hessian satisfies the low-rank property, Muon demonstrates a significant convergence speed advantage over SGD~\cite{shen2025convergence}. Using a Lyapunov function, it has been proven that Muon with momentum and decoupled weight decay is stable and converges to the KKT point of the implicitly constrained problem~\cite{chen2025muon}. The Muon optimizer is built on Newton--Schulz iterative updates, and Bernstein et al.~\cite{bernstein2025modular} constructed dual mappings for general network structures. This analytical framework can also be extended to the Muon optimizer. Newhouse et al.~\cite{newhouse2025training} proposed training a Transformer with Lipschitz constraints, employing constraint techniques such as Muon and soft cap. Although Chen et al.~\cite{chen2025muon} analyze Muon's convergence behavior under spectral norm constraints in a general nonconvex optimization setting, its effectiveness under AT has not been fully studied.
\section{Mapping of Theory and Algorithm}
\par Algorithm~\ref{alg:muon-step} illustrates the inline mechanism of \texttt{Opt.step()} described in Algorithm~\ref{alg:robust_training}.
\begin{algorithm}[!t]
\caption{One Muon Step under AT
  \textmd{(expands line~10 of Alg.\,1 when $\mathrm{Opt}=\mathrm{Muon}$;
  Newton--Schulz iteration inlined as sub-steps of S2).}}
\label{alg:muon-step}
\begin{algorithmic}[1]
\REQUIRE Weight block $W\in\mathbb{R}^{m\times n}$,
  momentum buffer $M$ (init.\ $\mathbf{0}$),
  gradient $G=\nabla_W L_{\mathrm{adv}}$,
  learning rate $\eta$, momentum $\mu$, weight decay $\lambda$,
  Newton-Schulz steps $s$

\STATE \slabel{S1}\ $M \gets \mu M + G$
  \hfill{\footnotesize$\triangleright$ \textit{momentum update (Assumption~\ref{assumption:adv_loss}:
    $G$ is the $\beta$-smooth AT gradient)}}

\STATE \slabel{S2}\ \textit{// Approximate polar factor
  --- supports \textbf{Lemma~\ref{lemma:polar_factor}}}
\STATE \hspace{2em} $X \gets M \;/\; \|M\|_F$
  \hfill{\footnotesize$\triangleright$ \textit{normalize, starts near Stiefel manifold}}
\FOR{\hspace{2em} $k = 1$ \textbf{to} $s$}
 \STATE \hspace{2em} 
$X \leftarrow \frac{3}{2}X-\frac{1}{2}XX^\top X$
\hfill{\footnotesize $\triangleright$ \textit{Newton--Schulz polar iteration}}
\ENDFOR
\STATE \hspace{2em} $\widetilde{U} \gets X$
  \hfill{\footnotesize$\triangleright$
    \textit{as in Theorem \ref{thm:upper_bound_spectral_norm}}}

\STATE \slabel{S3}\ $W \gets (1-\eta\lambda)\,W - \eta\,\widetilde{U}$
  \hfill{\footnotesize$\triangleright$;
    \textit{orthogonalised weight update (Theorem~\ref{thm:upper_bound_spectral_norm}:
    $\|W_t\|_2\le\max(\|W_0\|_2,\,\tfrac{1+\bar\varepsilon}{\lambda}),\ \text{for all}\ t \geq 0$)}}

\RETURN updated $W$,\ updated $M$
\end{algorithmic}
\end{algorithm}
\par Table~\ref{tab:algorithm_theory_mapping} maps theoretical results with algorithmic operations of the Newton--Schulz sub-routine defined in Algorithm~\ref{alg:muon-step}.
\begin{table}
\centering
\caption{Mapping of algorithmic operations and theoretical results.}
\label{tab:algorithm_theory_mapping}
\resizebox{\linewidth}{!}{
\begin{tabular}{clll}
\toprule
\textbf{Step} & \textbf{Operation} & \textbf{Theory} & \textbf{Implication}\\
\midrule
S1
  & $M\gets\mu M+G$
  & Assumption \ref{assumption:adv_loss}
  & $\beta$-smooth adversarial loss\\
S2 (init)
  & $X\gets M/\|M\|_F$
  & Lemma~\ref{lemma:polar_factor} (pre-cond.)
  & Normalizes scales \\
S2 (loop)
  & $X\gets\tfrac{3}{2}X-\tfrac{1}{2}XX^\top X$
  & Theorem \ref{thm:upper_bound_spectral_norm} (Newton--Schulz approximation and tolerance condition)
  & $\|\widetilde{U}\|_2\approx 1$\\
S3
  & $W\gets(1{-}\eta\lambda)W-\eta\widetilde{U}$
  & Theorem~\ref{thm:upper_bound_spectral_norm}
  & Bounded $\|W_t\|_2$ trajectory\\
S3 (all layers)
  & Multi-norm $\Delta_{\mathrm{union}}$
  & Theorem~\ref{thm:union_robustness}
  & Union robustness analysis\\
S3
  & AT loss descent
  & Theorem~\ref{thm:adv_loss_refined}
  & Nuclear-norm descent $\|M\|_*$\\
\bottomrule
\end{tabular}
}
\end{table}

\section{Supplementary Theoretical Remarks}
\label{sec:remark}
\begin{remark}[Interpretation of Proposition \ref{prop:dual_gradients}]
\label{remark:dual_gradients}
We acknowledge that true adversarial robustness depends on the global geometry of the model's decision boundary, not merely on local gradient magnitudes. The first-order characterization adopted here reveals robustness to the dual norm of the input gradient. It is a standard proxy used throughout the certification literature~\cite{hein2017formal, tsuzuku2018lipschitz}. It is exact for linear models and provides a tight bound in locally linear regions of ReLU networks. Although it does not capture higher-order curvature effects (e.g., decision boundary curvature exploited by second-order attacks), it remains predictive in practice.
\end{remark}
\begin{remark}[Interpretation of Proposition \ref{prop:bound_jacobian}]
For residual architectures, Proposition \ref{prop:bound_jacobian} can be applied blockwise. A block 
$h_{\ell+1}=h_\ell+F_\ell(h_\ell)$ satisfies
$\|J h_{\ell+1}\|_2\le 1+\|J F_\ell(h_\ell)\|_2$, and hence the network Jacobian is bounded by the product of such residual-block factors. Therefore, the sequential feedforward bound is used as a clean base case, while ResNet/WRN-style architectures follow by replacing each layer factor with its corresponding residual-block Lipschitz factor.
\end{remark}
\begin{remark}[Interpretation of Theorem~\ref{thm:upper_bound_spectral_norm}]
Theorem~\ref{thm:upper_bound_spectral_norm} says that Muon's polar update admits a spectral-norm stability ceiling: the weight decay term $\eta\lambda$ pulls the spectral norm toward zero, while the orthonormalized update
adds at most $\eta(1+\bar{\varepsilon})$ per step. These two forces
balance at the fixed point $\|W\|_2 = (1+\bar{\varepsilon})/\lambda$,
giving a hard ceiling that depends only on the learning rate, weight
decay, and approximation quality of Newton--Schulz iterations. Theorem~\ref{thm:upper_bound_spectral_norm} is a worst-case stability bound on the trajectory induced by bounded operator-norm updates. It should not be read as predicting that Muon yields smaller empirical weight spectral norms than SGD. In practice, SGD can yield even smaller spectral norms under AT. The theoretical content is that Muon's polar update cannot induce unbounded spectral growth. It is a property that distinguishes it from adaptive optimizers like AdamW, where we empirically observe such growth. The bound is therefore read as a sufficient condition for spectral stability. This theorem asserts that the polar update guarantees $\left\|\mathrm{W}_{t}\right\|_{2} \leqslant \max \left(\left\|\mathrm{W}_0\right\|_{2},(1+\bar{\varepsilon}) / \lambda\right)$, preventing unbounded spectral growth. This theorem does not claim the bound is tight. 
\end{remark}
\begin{remark}[Interpretation of Theorem~\ref{thm:union_robustness}]
The quantity $\hat{L}_p(\theta)$ is a computable upper bound on the
true $\ell_p$-Lipschitz constant $L_p(\theta)$,
obtained by chaining (i) the Jacobian–spectral bound of Proposition~\ref{prop:bound_jacobian}, (ii) the per-layer spectral bound $\left\|W_{\ell}\right\|_{2} \leq B_{\ell}$ from Theorem~\ref{thm:upper_bound_spectral_norm}, and (iii) the norm-equivalence $\|\delta\|_{2} \leq c_{p}\|\delta\|_{p}$. Theorem~\ref{thm:union_robustness} should be read as a structural statement rather than a usable robustness certificate. Specifically, it reflects that Muon's polar update in the context of empirical defense prevents unbounded spectral (hence Lipschitz) growth along the trajectory.
\end{remark}
\begin{remark}[Interpretation of Theorem~\ref{thm:adv_loss_refined}]
Theorem~\ref{thm:adv_loss_refined} shows that Muon's update provides a descent guarantee in the robust min-max optimization setting. Since the nuclear norm $\|M\|_* = \sum_i \sigma_i(M)$ accumulates all singular values of the gradient, the guaranteed descent is structurally larger than that of a raw gradient step, while the rank-bounded curvature term prevents the oscillations commonly observed with SGD under the high-variance adversarial loss landscape. In practice, Muon consistently achieves lower loss earlier and with less oscillation than other optimizers like SGD and AdamW across all architectures and norm settings, with the effect being most pronounced on ViT, where polar orthogonalization more effectively controls the rank structure of attention weight gradients.
\end{remark}
\section{Proof of Theoretical Analysis}
\label{sec:proof}
\subsection{Proof of Lemma~\ref{lemma:polar_factor}}
\label{subsec:proof:lemma:polar_factor}
\begin{proof}
\label{proof:polar_factor}
Let $U=\operatorname{Ortho}(A)=PV^\top$, where
$A=P\Sigma V^\top$ is the compact SVD of $A$. If $m\ge n$, then $P^\top P=I_n$ and $V$ is orthogonal. Hence
\[
U^\top U=VP^\top PV^\top=I_n .
\]
Moreover, since $A^\top A=V\Sigma^2V^\top$, we have
\[
U(A^\top A)^{1/2}
=
PV^\top(V\Sigma V^\top)
=
P\Sigma V^\top
=
A .
\]
If $m<n$, then $P$ is orthogonal and $V^\top V=I_m$. Hence
\[
UU^\top=PV^\top VP^\top=I_m .
\]
Moreover, since $AA^\top=P\Sigma^2P^\top$, we have
\[
(AA^\top)^{1/2}U
=
(P\Sigma P^\top)(PV^\top)
=
P\Sigma V^\top
=
A .
\]
Thus, all nonzero singular values of $U$ are equal to one, and therefore
$\|U\|_2=1$.

It remains to prove the projection statement. Let $\mathcal{S}$ denote
the corresponding rectangular Stiefel set, i.e.,
$\mathcal{S}=\{Q:Q^\top Q=I_n\}$ if $m\ge n$, and
$\mathcal{S}=\{Q:QQ^\top=I_m\}$ if $m<n$. For any $Q\in\mathcal{S}$,
\[
\|A-Q\|_F^2
=
\|A\|_F^2+\|Q\|_F^2-2\langle A,Q\rangle .
\]
Since $\|Q\|_F^2=\min\{m,n\}$ is fixed on $\mathcal{S}$, minimizing
$\|A-Q\|_F$ is equivalent to maximizing $\langle A,Q\rangle$. By von Neumann's
trace inequality,
\[
\langle A,Q\rangle
\le
\sum_i \sigma_i(A)\sigma_i(Q)
=
\sum_i \sigma_i(A),
\]
because every $Q\in\mathcal{S}$ has all nonzero singular values equal to one.
For $Q=U=PV^\top$, equality holds:
\[
\langle A,U\rangle
=
\operatorname{tr}(A^\top U)
=
\operatorname{tr}(\Sigma)
=
\sum_i \sigma_i(A).
\]
Hence $U$ is a minimizer of $\min_{Q\in\mathcal{S}}\|A-Q\|_F$. The proof is complete.
\end{proof}
\subsection{Proof of Theorem~\ref{thm:upper_bound_spectral_norm}}
\label{subsec:proof:thm:upper_bound_spectral_norm}
\begin{proof}
Using the triangle inequality and absolute homogeneity of the spectral norm, we have
\begin{equation}
\left\|W_{t+1}\right\|_{2}=\left\|(1-\eta \lambda) W_{t}-\eta \widetilde{U}_{t}\right\|_{2} \leq(1-\eta \lambda)\left\|W_{t}\right\|_{2}+\eta\left\|\widetilde{U}_{t}\right\|_{2}.
\end{equation}
Applying the assumption of Eq.~(\ref{eq:muon_trajectory_assumption}) with $a_{t}=\left\|W_{t}\right\|_{2}$, we obtain
\begin{equation}
a_{t+1} \leq(1-\eta \lambda) a_{t}+\eta+\eta \varepsilon_{t}.
\end{equation}
Unrolling the recursion gives
\begin{equation}
a_{t} \leq(1-\eta \lambda)^{t} a_{0}+\eta \sum_{k=0}^{t-1}(1-\eta \lambda)^{t-1-k}+\eta \sum_{k=0}^{t-1}(1-\eta \lambda)^{t-1-k} \varepsilon_{k}.
\end{equation}
The second term is a geometric series
\begin{equation}
\eta \sum_{k=0}^{t-1}(1-\eta \lambda)^{t-1-k}=\eta \sum_{j=0}^{t-1}(1-\eta \lambda)^{j}=\eta \cdot \frac{1-(1-\eta \lambda)^{t}}{\eta \lambda}=\frac{1-(1-\eta \lambda)^{t}}{\lambda}.
\end{equation}
Substituting back yields Eq.~(\ref{eq:trajectory_upper_bound}). If $\varepsilon_{k} \leq \bar{\varepsilon}$, then
\begin{equation}
\eta \sum_{k=0}^{t-1}(1-\eta \lambda)^{t-1-k} \varepsilon_{k} \leq \eta \bar{\varepsilon} \sum_{j=0}^{t-1}(1-\eta \lambda)^{j}=\bar{\varepsilon} \cdot \frac{1-(1-\eta \lambda)^{t}}{\lambda},
\end{equation}
which gives Eq.~(\ref{eq:eq:trajectory_inequality}). The proof is complete.
\end{proof}
\subsection{Proof of Theorem~\ref{thm:union_robustness}}
\label{subsec:proof:thm:union_robustness}
\begin{proof}
\textbf{Step 1: From Jacobian spectral norms to an $\ell_2$-Lipschitz bound on the margin.}

For each competing class $j\ne y$, define the pairwise margin
\begin{equation}
m_{y,j}(x):=g_\theta(x)_y-g_\theta(x)_j
=
\langle e_y-e_j,\,g_\theta(x)\rangle .
\end{equation}
Here $e_y-e_j$ is fixed and satisfies $\|e_y-e_j\|_2=\sqrt{2}$. Hence
\begin{equation}
\nabla_x m_{y,j}(x)
=
J_{g_\theta}(x)^\top(e_y-e_j),
\end{equation}
and therefore, it satisfies
\begin{equation}
\|\nabla_x m_{y,j}(x)\|_2
\le
\sqrt{2}\,\|J_{g_\theta}(x)\|_2
\le
\sqrt{2}\prod_{\ell=1}^{L}\|W_\ell\|_2,
\end{equation}
where the last inequality follows from Proposition~\ref{prop:bound_jacobian}
assuming $1$-Lipschitz activations. Otherwise, the product is replaced by
$L_\phi^{L-1}\prod_\ell \|W_\ell\|_2$.

Since the multiclass margin can be written as
\begin{equation}
m_\theta(x,y)
=
g_\theta(x)_y-\max_{j\ne y}g_\theta(x)_j
=
\min_{j\ne y}m_{y,j}(x),
\end{equation}
and the pointwise minimum of finitely many $L$-Lipschitz functions is still
$L$-Lipschitz, $m_\theta(\cdot,y)$ is $L_2(\theta)$-Lipschitz with respect to
$\|\cdot\|_2$, with
\begin{equation}
L_2(\theta)
\le
C_{\rm marg}\prod_{\ell=1}^{L}\|W_\ell\|_2,
\qquad
C_{\rm marg}:=\max_{j\ne y}\|e_y-e_j\|_2=\sqrt{2}.
\label{eq:D31}
\end{equation}

\textbf{Step 2: Spectral-norm control from Muon.}

By Theorem~\ref{thm:upper_bound_spectral_norm}, running Muon with learning rate $\eta$ and weight decay $\lambda > 0$
guarantees
\begin{equation}
\|W_\ell\|_2 \;\le\; B_\ell
\qquad \text{for all } \ell \in \{1, \dots, L\} \text{ and all training steps},
\end{equation}
where $B_\ell = \max\!\bigl(\|W_\ell^{(0)}\|_2,\, (1+\bar{\varepsilon})/\lambda\bigr)$.
Combining with Eq. \eqref{eq:D31} yields
\begin{equation}
L_2(\theta)
\le
C_{\rm marg}\prod_{\ell=1}^{L} B_\ell .
\label{eq:D32}
\end{equation}

\textbf{Step 3: From $\ell_2$-Lipschitzness to $\ell_p$-Lipschitzness via norm equivalence.}

For each $p \in \{1, 2, \infty\}$, let $c_p$ denote the norm-equivalence constant satisfying
\begin{equation}
\|\delta\|_2 \;\le\; c_p\, \|\delta\|_p,
\qquad \text{for all } \delta \in \mathbb{R}^d,
\label{eq:proof_union_robustness_step3}
\end{equation}
with $c_\infty = \sqrt{d}$, $c_2 = 1$, and $c_1 = 1$. Then for any
$\delta \in \mathbb{R}^d$ and any $p \in \mathcal{P}$,
\begin{equation}
\bigl| m_\theta(x + \delta, y) - m_\theta(x, y) \bigr|
\le
L_2(\theta)\,\|\delta\|_2
\le
L_2(\theta)\, c_p\, \|\delta\|_p
\le
\hat{L}_p(\theta)\, \|\delta\|_p,
\label{eq:D33}
\end{equation}
where
\[
\hat{L}_p(\theta)
:=
c_p\, C_{\rm marg}\prod_{\ell=1}^{L}B_\ell
\]
is the computable upper bound from Eq.~(\ref{eq:union_robustness_definition}), and the last inequality uses Eq. \eqref{eq:D32}.

\textbf{Step 4: Union robustness.}

Take any $\delta \in \Delta_{\mathrm{union}}
= \bigcup_{p \in \mathcal{P}} \{\delta : \|\delta\|_p \le \varepsilon_p\}$.
Then there exists some $p \in \mathcal{P}$ with $\|\delta\|_p \le \varepsilon_p$.
By \eqref{eq:D33},
\begin{equation}
m_\theta(x + \delta, y)
\ge
m_\theta(x, y) - \hat{L}_p(\theta)\, \|\delta\|_p
\ge
m_\theta(x, y) - \hat{L}_p(\theta)\, \varepsilon_p
\ge
m_\theta(x, y) - \max_{p' \in \mathcal{P}} \hat{L}_{p'}(\theta)\, \varepsilon_{p'}.
\end{equation}
Under the assumption of Theorem~\ref{thm:union_robustness},
$m_\theta(x, y) > \max_{p' \in \mathcal{P}} \hat{L}_{p'}(\theta)\, \varepsilon_{p'}$,
hence $m_\theta(x + \delta, y) > 0$. Since $\delta \in \Delta_{\mathrm{union}}$ was arbitrary,
$m_\theta(x + \delta, y) > 0$ for all $\delta \in \Delta_{\mathrm{union}}$. Thus, the claim follows.
\end{proof}
\subsection{Proof of Theorem~\ref{thm:adv_loss_refined}}
\label{subsec:proof:thm:adv_loss_refined}
\begin{proof}
Let $M$ denote the momentum-accumulated gradient used by Muon, and let $\widetilde U$ denote the inexact Newton--Schulz output satisfying
\begin{equation}
\|\widetilde U - \operatorname{Ortho}(M)\|_F \le \delta_{\mathrm{orth}},
\qquad
\|\widetilde U\|_2 \le 1+\varepsilon,
\qquad
\|\widetilde U\|_F^2 \le r(1+\varepsilon)^2,
\label{eq:D41-assumption}
\end{equation}
where $r=\operatorname{rank}(M)$, or any upper bound on it. Consider the actual
update $W^+ = W-\eta\,\widetilde U$. By Assumption~\ref{assumption:adv_loss}
($\beta$-smoothness of $F$ w.r.t.\ $\|\cdot\|_F$), we have
\begin{equation}
F(W^+)
\le
F(W)
-\eta\langle G,\widetilde U\rangle
+\frac{\beta}{2}\eta^2\|\widetilde U\|_F^2 .
\label{eq:D41-star}
\end{equation}

We decompose the inner product as
\begin{equation}
\langle G,\widetilde U\rangle
=
\underbrace{\langle M,\operatorname{Ortho}(M)\rangle}_{\textnormal{(i)}}
+
\underbrace{\langle G-M,\widetilde U\rangle}_{\textnormal{(ii)}}
+
\underbrace{\langle M,\widetilde U-\operatorname{Ortho}(M)\rangle}_{\textnormal{(iii)}} .
\label{eq:D41-decomp}
\end{equation}

\textbf{(i) Ideal inner-product term.}
By the same SVD argument as in the idealized case
(Lemma~\ref{lemma:polar_factor}), the polar factor satisfies
\begin{equation}
\langle M,\operatorname{Ortho}(M)\rangle
=
\|M\|_*,
\label{eq:D41-ideal-term}
\end{equation}
where $\|M\|_*=\sum_{i=1}^{r}\sigma_i(M)$ is the nuclear norm.

\textbf{(ii) Gradient--momentum mismatch.}
By H\"older's inequality for Schatten norms, i.e., the duality between
the nuclear norm and the spectral norm, we have
\begin{equation}
\langle G-M,\widetilde U\rangle
\ge
-\|G-M\|_*\,\|\widetilde U\|_2
\ge
-(1+\varepsilon)\|G-M\|_* .
\label{eq:D41-mismatch}
\end{equation}

\textbf{(iii) Orthogonalization error.}
By the Cauchy--Schwarz inequality for the Frobenius inner product,
\begin{equation}
\begin{aligned}
\bigl|
\langle M,\widetilde U-\operatorname{Ortho}(M)\rangle
\bigr|
&\le
\|M\|_F
\|\widetilde U-\operatorname{Ortho}(M)\|_F  \\
&\le
\|M\|_F\,\delta_{\mathrm{orth}}  \\
&\le
\sqrt r\,\|M\|_2\,\delta_{\mathrm{orth}},
\end{aligned}
\label{eq:D41-orth-error}
\end{equation}
where the last step uses $\|M\|_F\le \sqrt r\,\|M\|_2$.

\textbf{Combining (i)--(iii).}
Substituting Eq. \eqref{eq:D41-ideal-term}, Eq. \eqref{eq:D41-mismatch}, and
Eq. \eqref{eq:D41-orth-error} into Eq. \eqref{eq:D41-decomp} gives
\begin{equation}
\langle G,\widetilde U\rangle
\ge
\|M\|_*
-
(1+\varepsilon)\|G-M\|_*
-
\sqrt r\,\|M\|_2\,\delta_{\mathrm{orth}} .
\label{eq:D41-inner-lower-bound}
\end{equation}
Plugging Eq. \eqref{eq:D41-inner-lower-bound} into \eqref{eq:D41-star} and using
$\|\widetilde U\|_F^2\le r(1+\varepsilon)^2$, we obtain
\begin{equation}
\begin{aligned}
F(W^+)
\le\;&
F(W)
-\eta\|M\|_*
+\eta(1+\varepsilon)\|G-M\|_*  \\
&\quad
+\eta\sqrt r\,\|M\|_2\,\delta_{\mathrm{orth}}
+\frac{\beta}{2}\eta^2 r(1+\varepsilon)^2 .
\end{aligned}
\end{equation}
This is exactly the descent inequality of
Theorem~\ref{thm:adv_loss_refined}. The proof is complete.
\end{proof}
\section{Spectral Analysis: Spectral Norms and Condition Numbers} 
\label{sec:spectral_norm}
\par As a lightweight diagnostic of the optimization-side mechanism of Muon on spectral conditioning and robustness, we track the mean per-layer condition number $\kappa = \sigma_{\max}/\sigma_{\min}$ every 5 epochs throughout AT on PreActResNet-18 (Figure~\ref{fig:condition_number_curve}). Muon maintains a low and stable $\kappa \approx 5$ throughout training, comparable to SGD, while AdamW's $\kappa$ grows rapidly to $\approx 50$ within the first 30 epochs. Correlating $\kappa$ with robust accuracy (Figure~\ref{fig:kappa_vs_robust}) reveals two distinct regimes: SGD and Muon concentrate in the low $\kappa$, high-robustness region, while AdamW occupies the high $\kappa$, lower-robustness region. This phenomenon suggests that what determines robustness may not be purely ``Jacobian scale'', but rather ``Jacobian geometry'', that is, whether the distribution of singular values is balanced. This supports the interpretation that spectral conditioning acts as a \emph{regularizer for the optimization dynamics} of AT. A well-conditioned weight matrix propagates gradients more uniformly during backpropagation, mitigating the vanishing/exploding gradient issues that otherwise destabilize AT. These diagnostics show that SGD often achieves the smallest absolute spectral norm, consistent with its strong robustness under AT. Muon does not uniformly dominate SGD in terms of raw spectral-norm magnitude. Instead, its spectral behavior is consistently closer to SGD than to AdamW, and it avoids the large spectral growth observed for AdamW in several settings. Therefore, our claim should be interpreted as the stability ceiling rather than the strict minimization of weight spectral norms. This also explains why SGD remains a very strong baseline, while Muon provides a more stable alternative to AdamW.
\begin{figure}[!t]
    \centering
   \includegraphics[width=0.6\hsize]{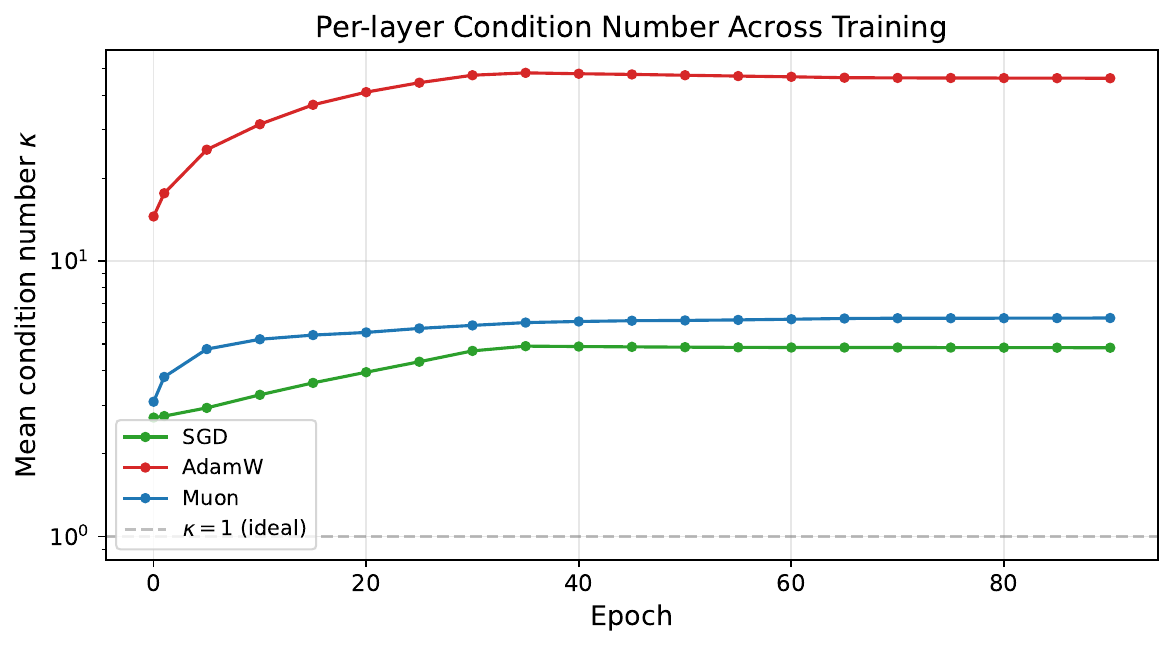}
    \caption{Per-layer condition number change during training.}
    \label{fig:condition_number_curve}
\end{figure}
\begin{figure}[!t]
    \centering
    \includegraphics[width=0.6\hsize]{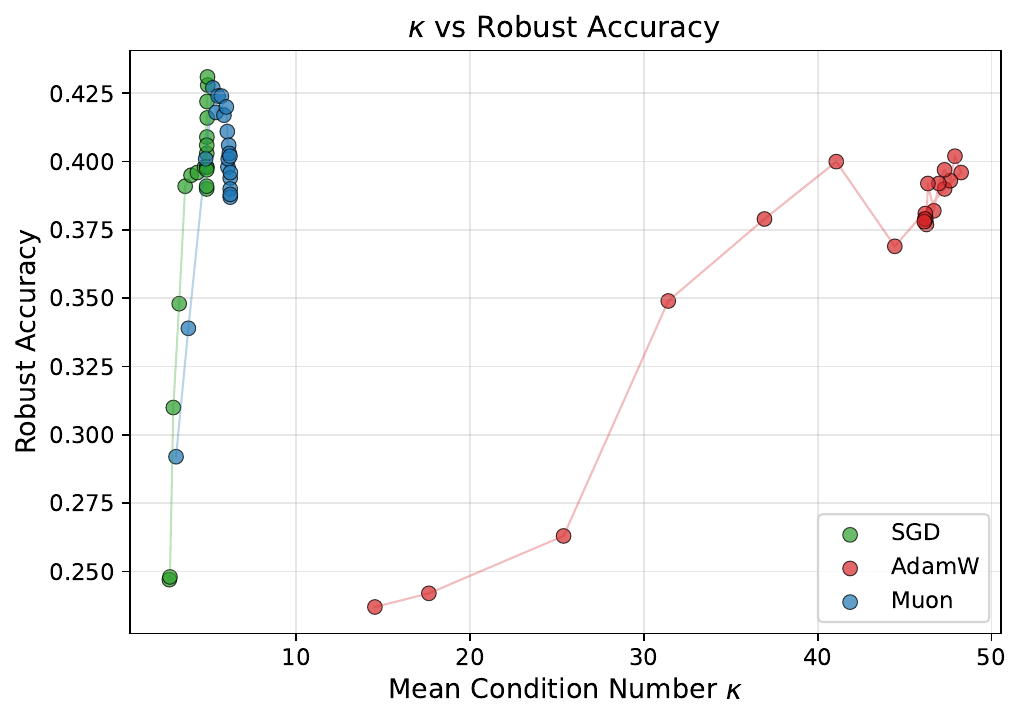}
    \caption{Spectral conditioning and robustness.}
    \label{fig:kappa_vs_robust}
\end{figure}
\par To probe the empirical tightness of the spectral-norm bound, we train PreActResNet-18, WRN-34-10, and ViT-B for 100 epochs on CIFAR-10 using the multi-norm $\ell_\infty \ + \ \ell_{1}$ pipeline (APGD-10 attack, piecewise LR schedule, $\lambda=5 \times 10^{-4}$). Fig. \ref{fig:spectral_norms_AT} shows per-epoch median and max $\| \mathrm{W}_ \mathrm{\ell}\|_{2}$ trajectories under each optimizer. Three observations support Theorem \ref{thm:upper_bound_spectral_norm}'s interpretation as a stability distinction rather than a tight bound. First, the ceiling holds universally with substantial slack. Across all 9 (model, optimizer) combinations, max $\| \mathrm{W}_ \mathrm{\ell}\|_{2}$ stays within a single decade (around 1–55), three to four orders of magnitude below the theoretical ceiling $\mathrm{B}=(1+\bar{\varepsilon}) / \lambda \approx 2200$. The bound is loose by design: it certifies that Muon's polar updates cannot induce unbounded spectral growth, not that $\| \mathrm{W}_ \mathrm{\ell}\|_{2}$ approaches $B$. Second, Muon and SGD plateau by epoch 35 on every architecture. Both optimizers exhibit the LR-schedule-driven plateau predicted by Theorem \ref{thm:upper_bound_spectral_norm} that spectral norms grow during the high-LR phase and stabilize once $\eta$ is decayed at epoch 33. The plateau is architecture-invariant in qualitative shape, though absolute values differ across model sizes. Third, AdamW's spectral behavior is regime-dependent. AdamW achieves the smallest spectral norms on PreActResNet-18, but the largest on WRN-34-10 and ViT-B. This non-monotonicity reflects the per-parameter scaling mechanism of AdamW: the update magnitude is governed by $m_t / (\sqrt{v_t} + \epsilon_\mathrm{AdamW})$, where $v_t$ is the gradient second-moment Exponential Moving Average (EMA). When the model is not large (PreActResNet-18), gradient directions are repetitive and $v_t$ saturates rapidly, making AdamW behave as approximately constant-step SGD with bounded updates. When the model is large enough that gradient statistics become heterogeneous across layers and directions (WRN-34-10), some directions develop near-zero $v_t$, amplifying the effective update without bound.

\begin{figure*}[!h]
\centering
\begin{subfigure}[t]{0.33\textwidth}
  \centering
  \includegraphics[width=\linewidth]{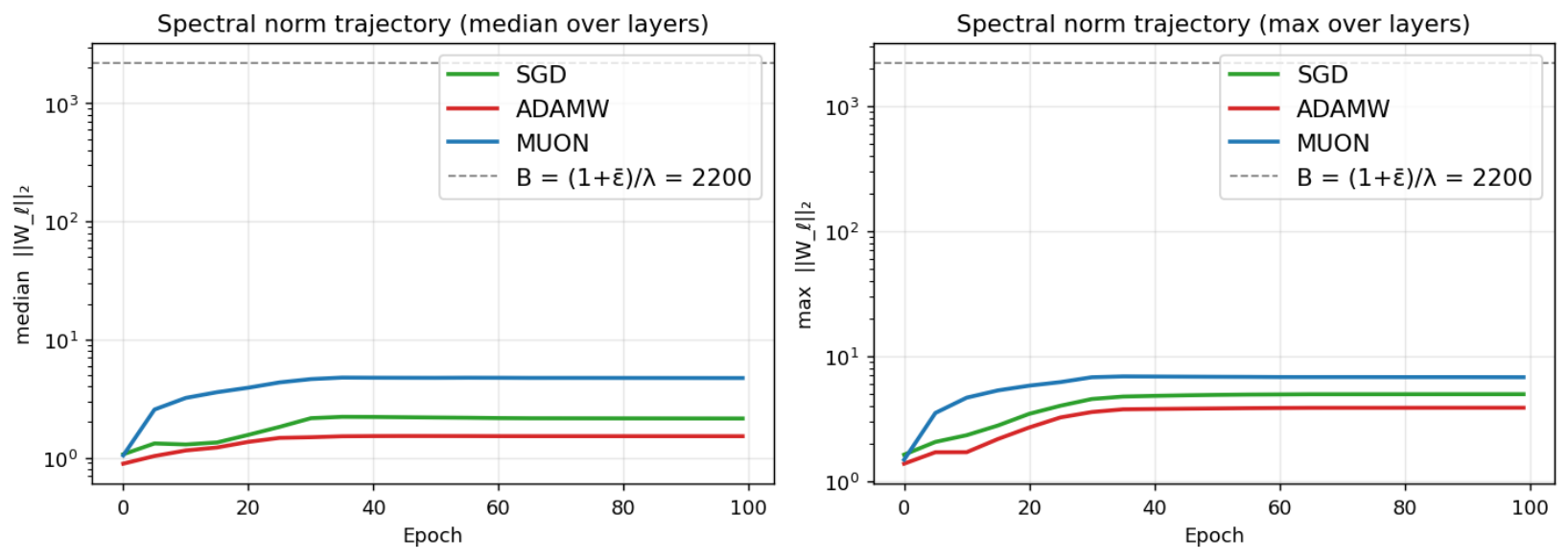}
  \caption{PreActResNet-18}
\end{subfigure}\hfill
\begin{subfigure}[t]{0.33\textwidth}
  \centering
  \includegraphics[width=\linewidth]{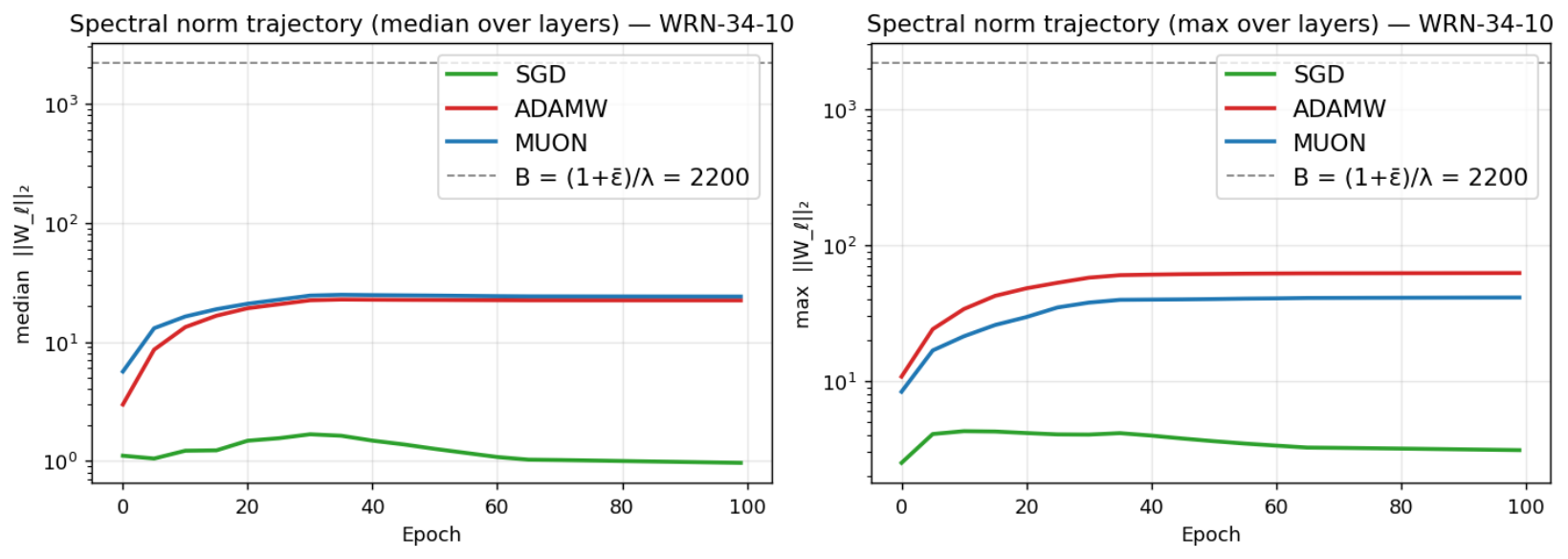}
  \caption{WRN-34-10}
\end{subfigure}\hfill
\begin{subfigure}[t]{0.33\textwidth}
  \centering
  \includegraphics[width=\linewidth]{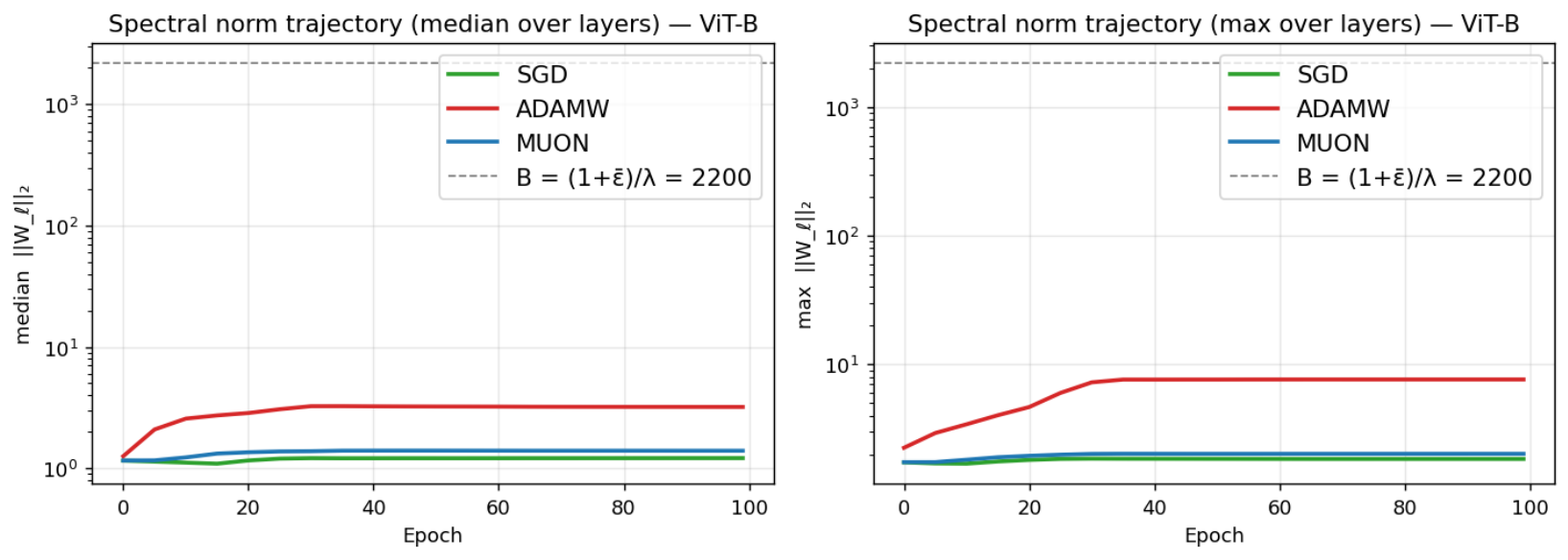}
  \caption{VIT-B}
\end{subfigure}
\caption{The spectral norm change with different optimizers during AT.}
\label{fig:spectral_norms_AT}
\end{figure*}
\section{Loss of Different Optimizers}
\label{sec:loss_fig}
\par Figures~\ref{fig:loss_preresnet}, \ref{fig:loss_wrn_34_10}, \ref{fig:loss_wrn_34_20}, \ref{fig:loss_vit_b}, and \ref{fig:loss_vit_l} depict the training loss curves during AT on CIFAR-10~\cite{krizhevsky2009learning} with different optimizers on PreActResNet-18~\cite{he2016identity}, WRN-34-10~\cite{zagoruyko2016wide}, WRN-34-20~\cite{zagoruyko2016wide}, ViT-B~\cite{dosovitskiy2021image}, and ViT-L~\cite{dosovitskiy2021image}. Across architectures, Muon converges faster and yields a smoother training-loss curve than both SGD and AdamW. These empirical observations are consistent with the nuclear-norm descent guarantee in Theorem~\ref{thm:adv_loss_refined}. In contrast, on the ViT architecture~\cite{dosovitskiy2021image}, it is difficult to use an adaptive optimizer such as AdamW for AT. When AdamW exhibits numerical instability on ViTs, we report the resulting evaluation accuracies in the main tables rather than marking the runs as ``failed''. Such cases are described as severe degradation or near-random performance when the reported robust accuracies are close to chance level.
\begin{figure*}[!h]
\centering
\begin{subfigure}[t]{0.245\textwidth}
  \centering
  \includegraphics[width=\linewidth]{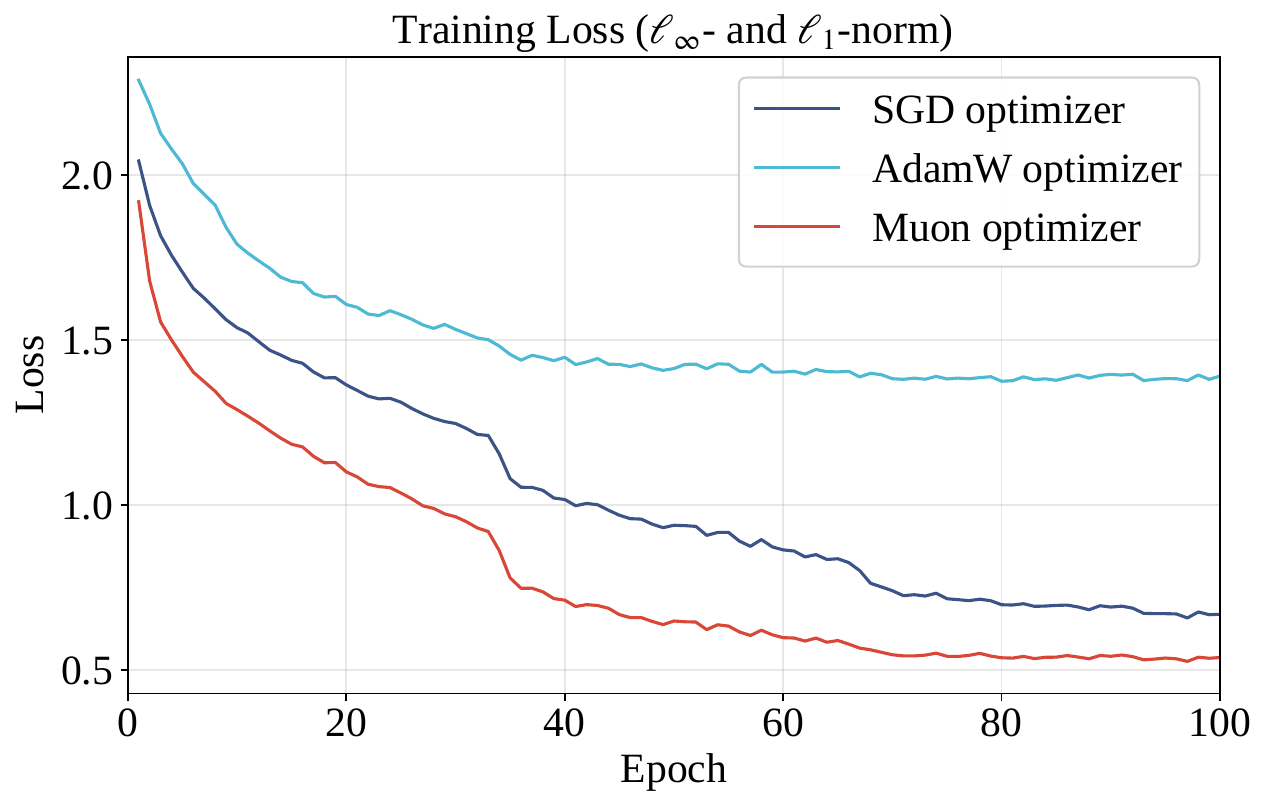}
  \caption{$\ell_\infty + \ell_1$}
\end{subfigure}\hfill
\begin{subfigure}[t]{0.245\textwidth}
  \centering
  \includegraphics[width=\linewidth]{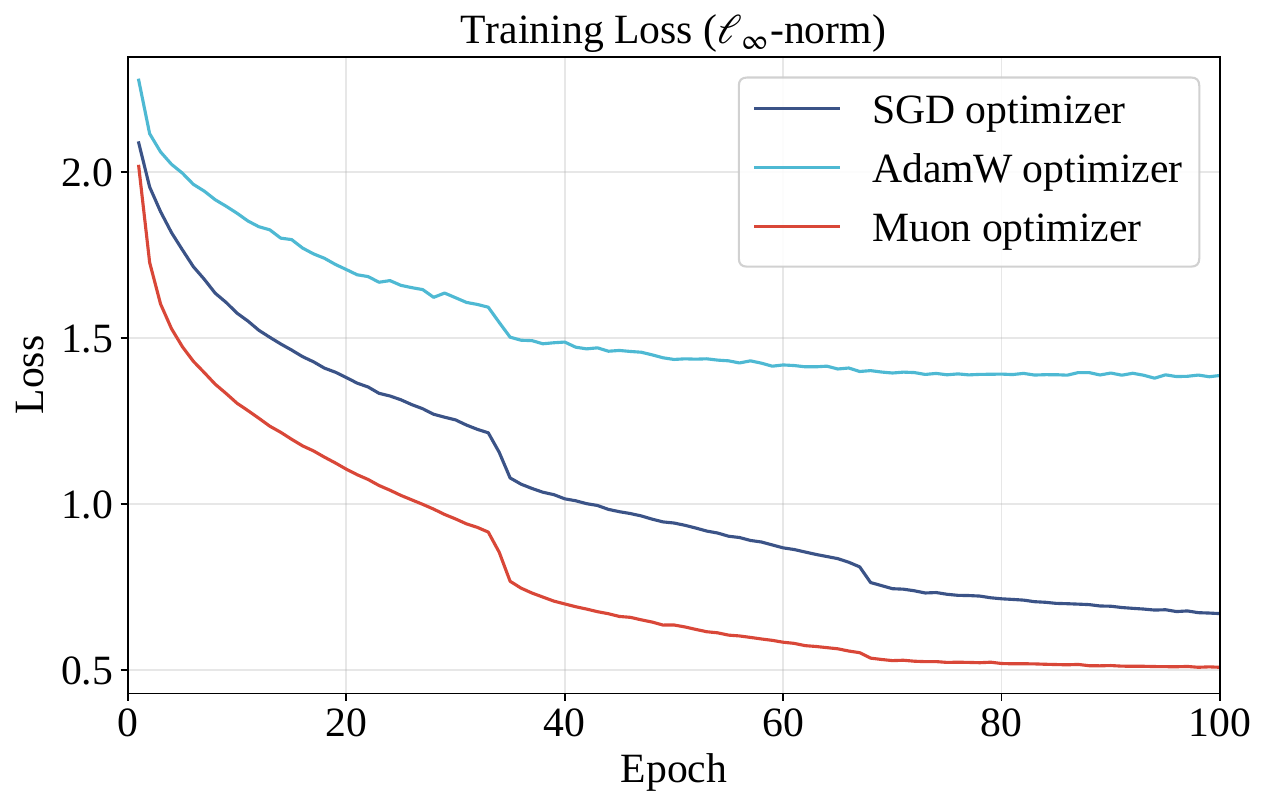}
  \caption{$\ell_\infty$}
\end{subfigure}\hfill
\begin{subfigure}[t]{0.245\textwidth}
  \centering
  \includegraphics[width=\linewidth]{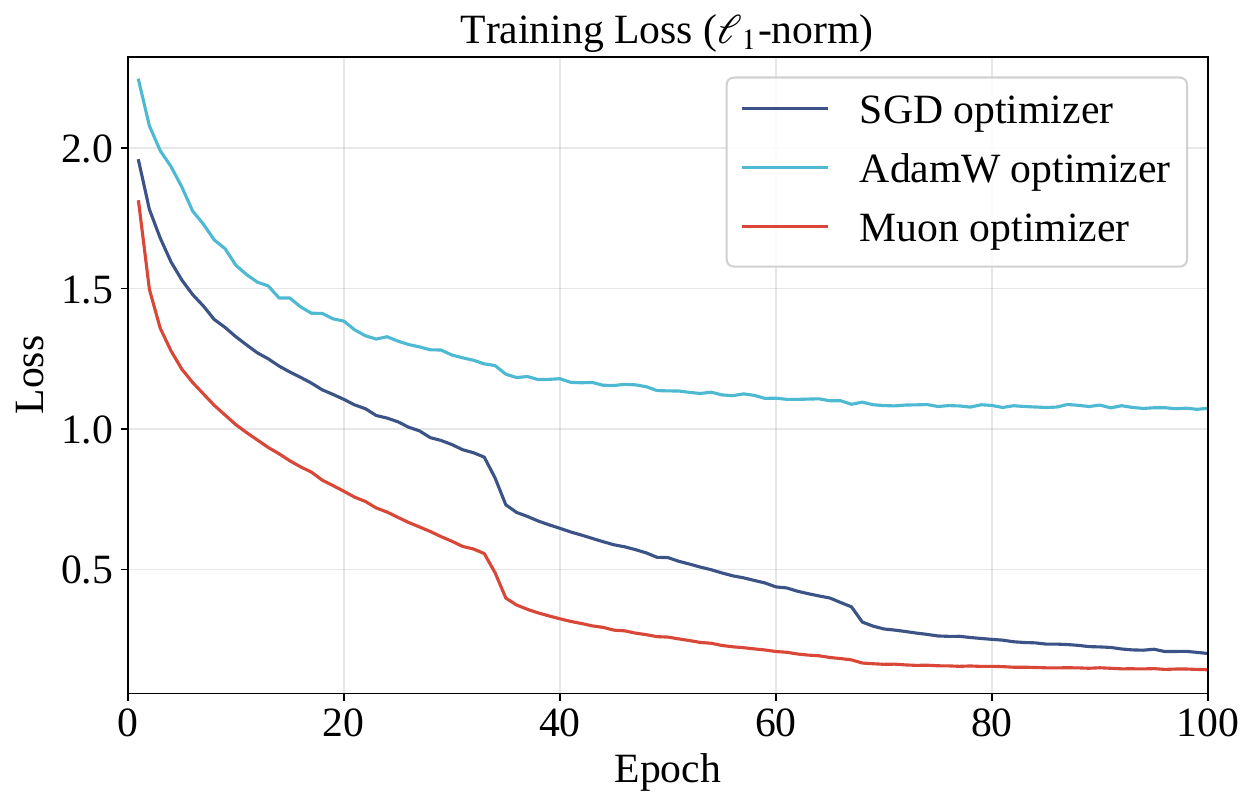}
  \caption{$\ell_1$}
\end{subfigure}\hfill
\begin{subfigure}[t]{0.245\textwidth}
  \centering
  \includegraphics[width=\linewidth]{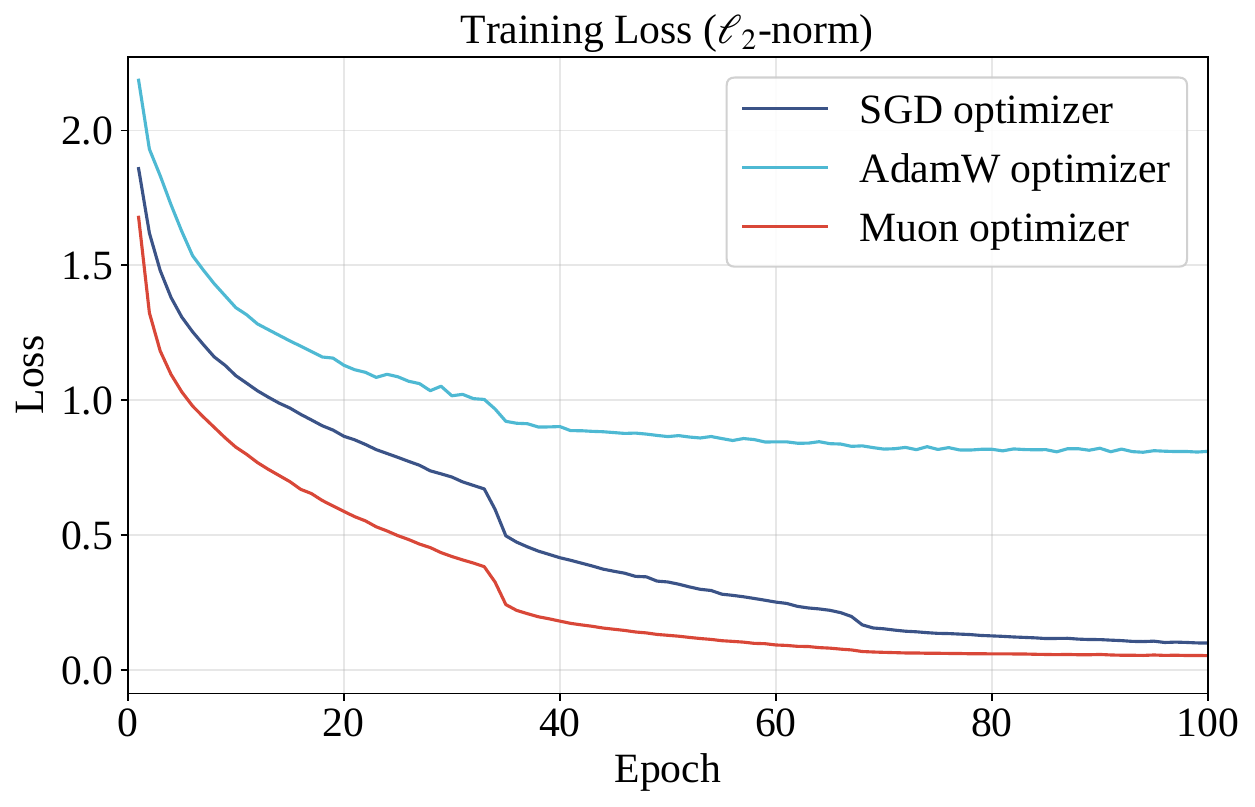}
  \caption{$\ell_2$}
\end{subfigure}
\caption{Training loss on PreActResNet-18~\cite{he2016identity}.}
\label{fig:loss_preresnet}
\end{figure*}
\begin{figure*}[!h]
\centering
\begin{subfigure}[t]{0.245\textwidth}
  \centering
  \includegraphics[width=\linewidth]{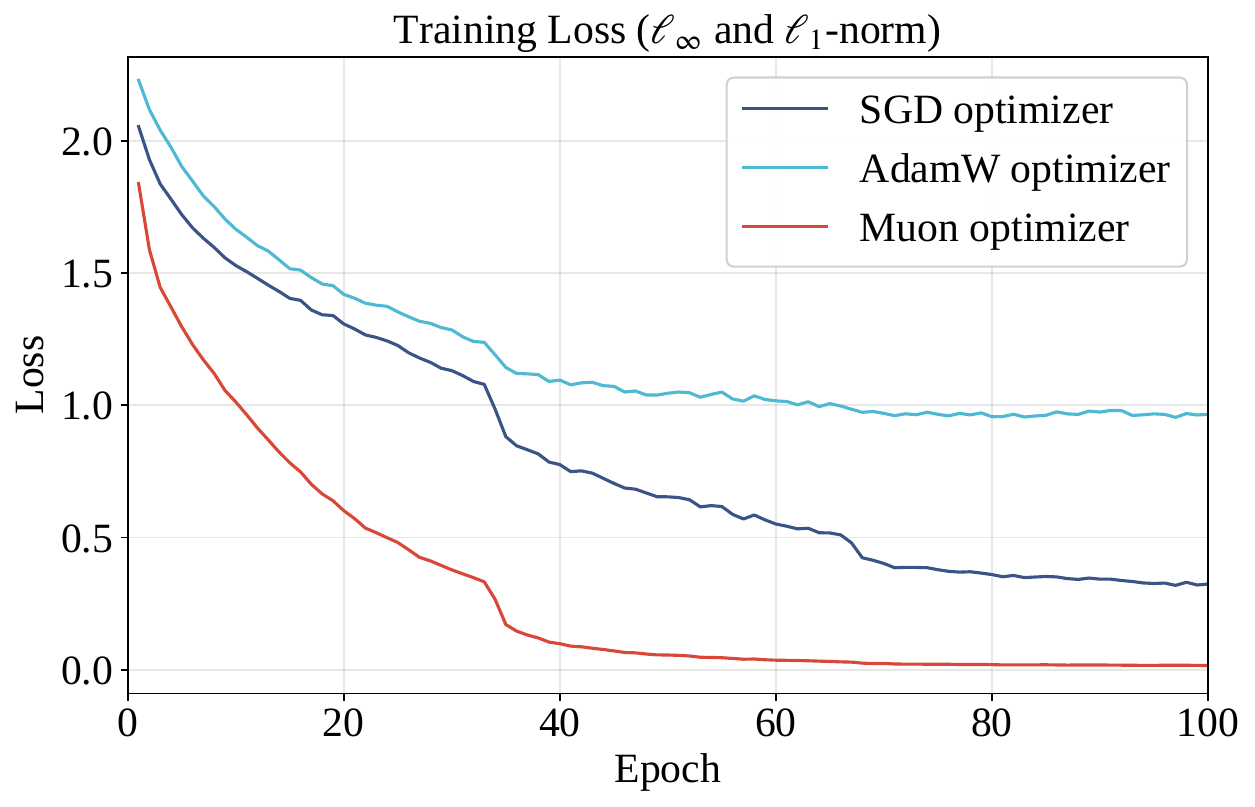}
  \caption{$\ell_\infty + \ell_1$}
\end{subfigure}\hfill
\begin{subfigure}[t]{0.245\textwidth}
  \centering
  \includegraphics[width=\linewidth]{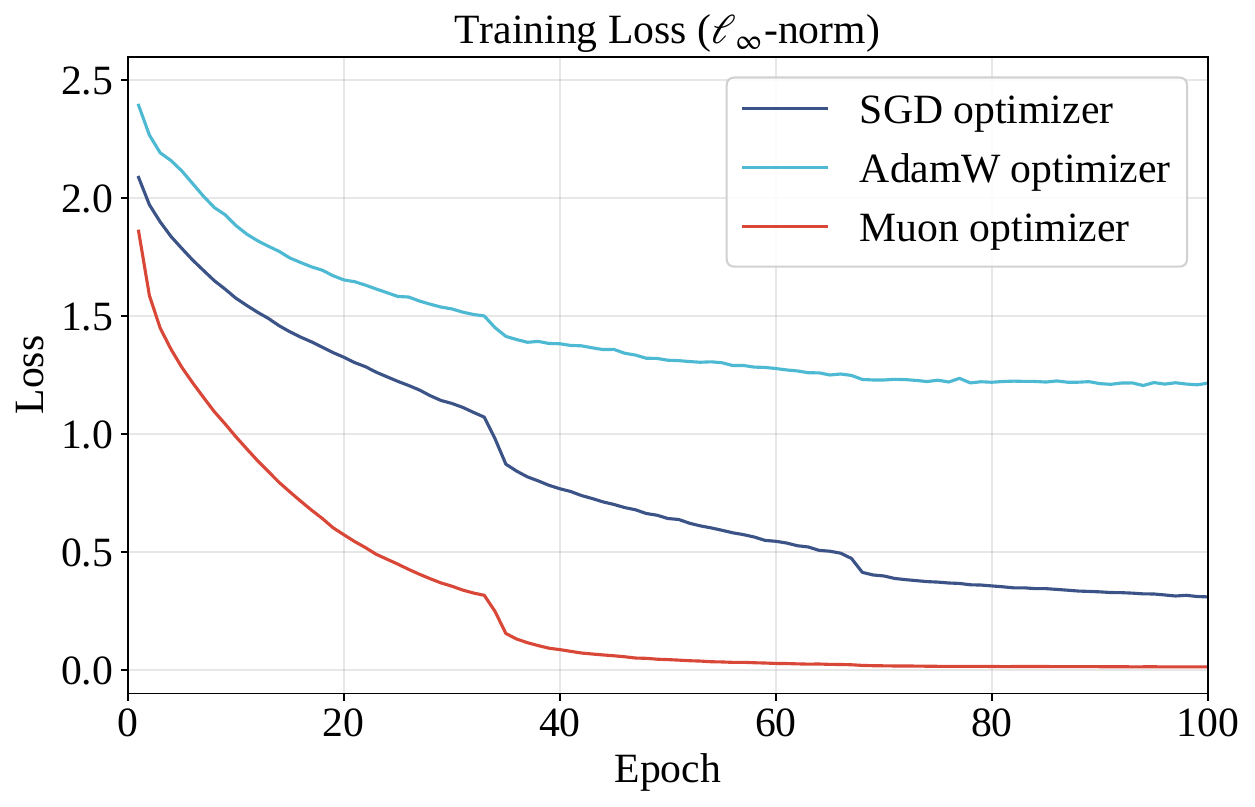}
  \caption{$\ell_\infty$}
\end{subfigure}\hfill
\begin{subfigure}[t]{0.245\textwidth}
  \centering
  \includegraphics[width=\linewidth]{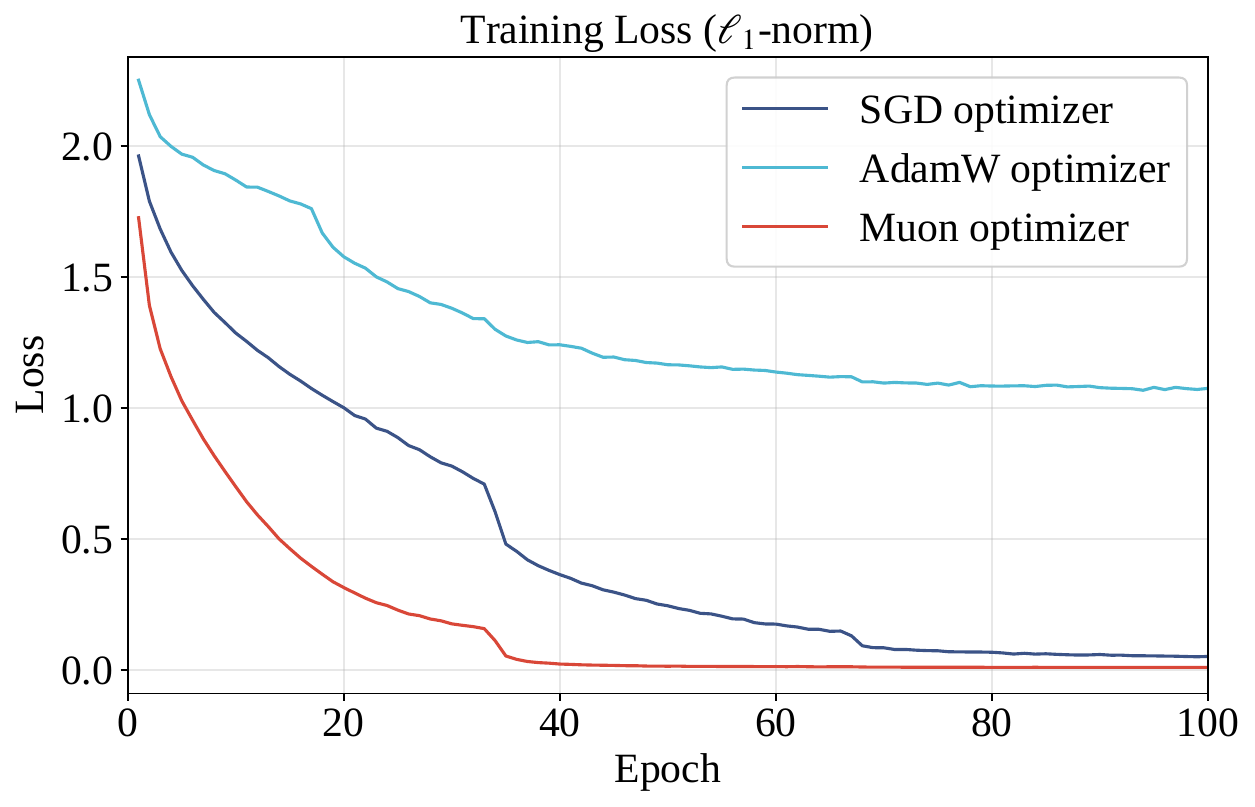}
  \caption{$\ell_1$}
\end{subfigure}\hfill
\begin{subfigure}[t]{0.245\textwidth}
  \centering
  \includegraphics[width=\linewidth]{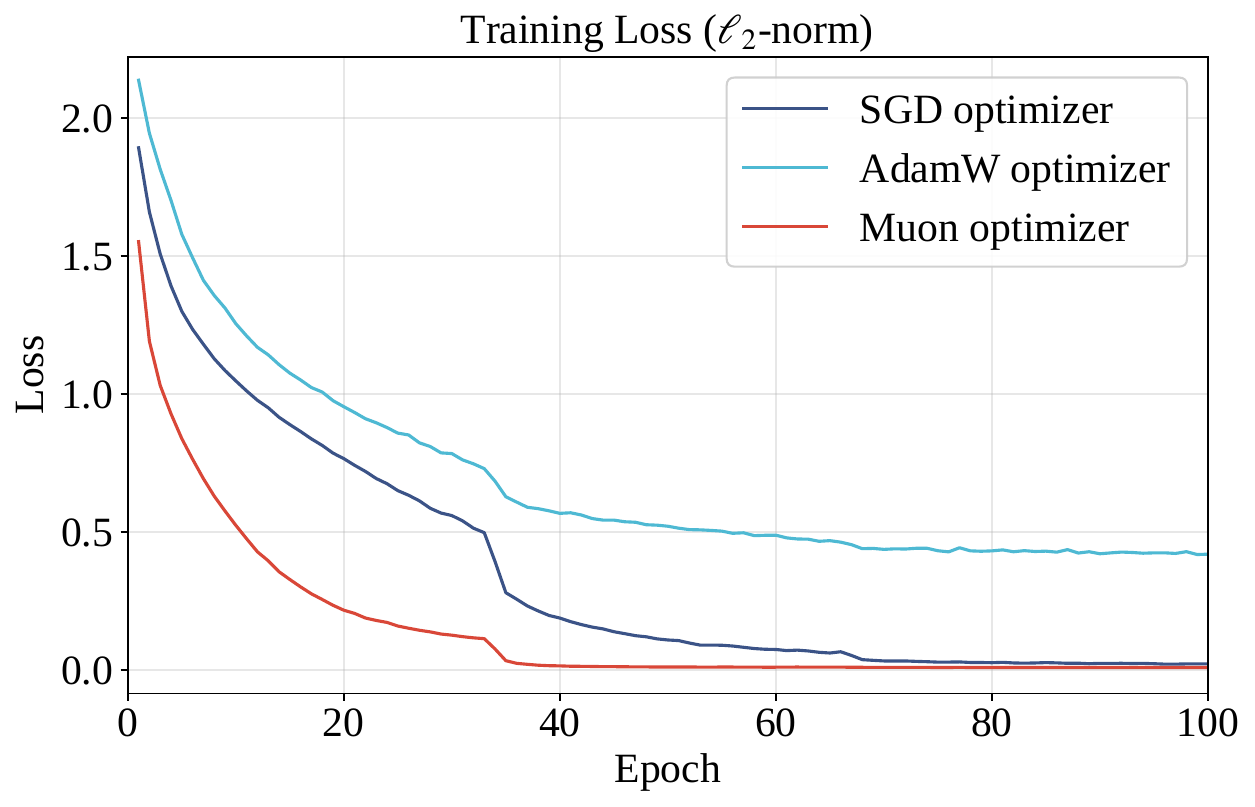}
  \caption{$\ell_2$}
\end{subfigure}
\caption{Training loss on WRN-34-10~\cite{zagoruyko2016wide}.}
\label{fig:loss_wrn_34_10}
\end{figure*}
\begin{figure*}[!t]
\centering
\begin{subfigure}[t]{0.245\textwidth}
  \centering
  \includegraphics[width=\linewidth]{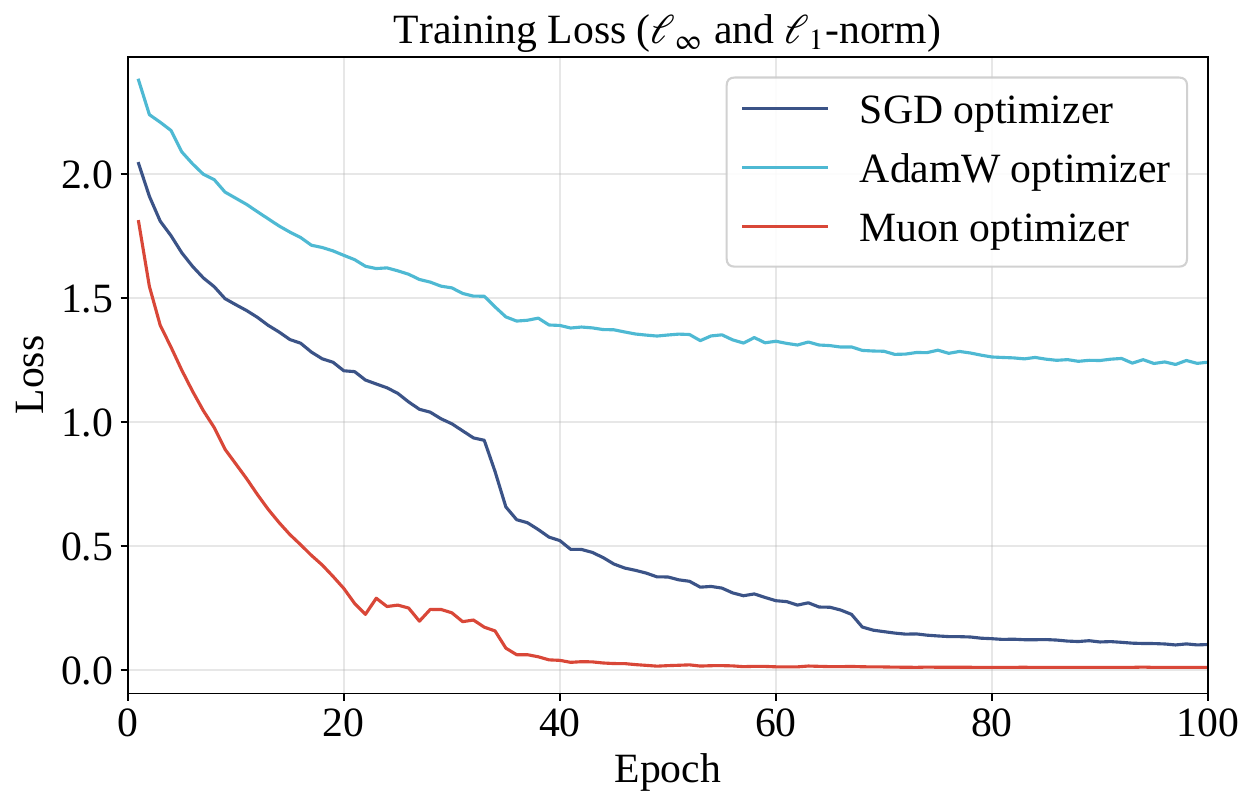}
  \caption{$\ell_\infty + \ell_1$}
\end{subfigure}\hfill
\begin{subfigure}[t]{0.245\textwidth}
  \centering
  \includegraphics[width=\linewidth]{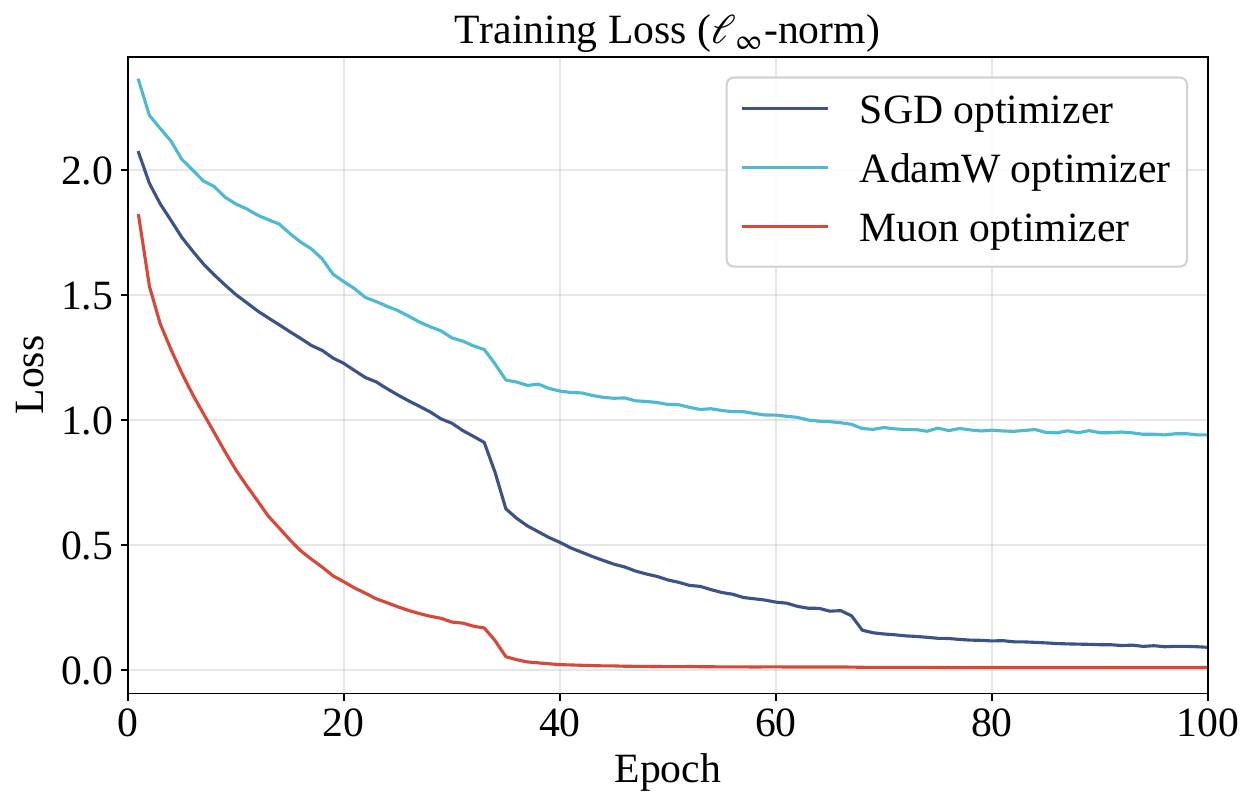}
  \caption{$\ell_\infty$}
\end{subfigure}\hfill
\begin{subfigure}[t]{0.245\textwidth}
  \centering
  \includegraphics[width=\linewidth]{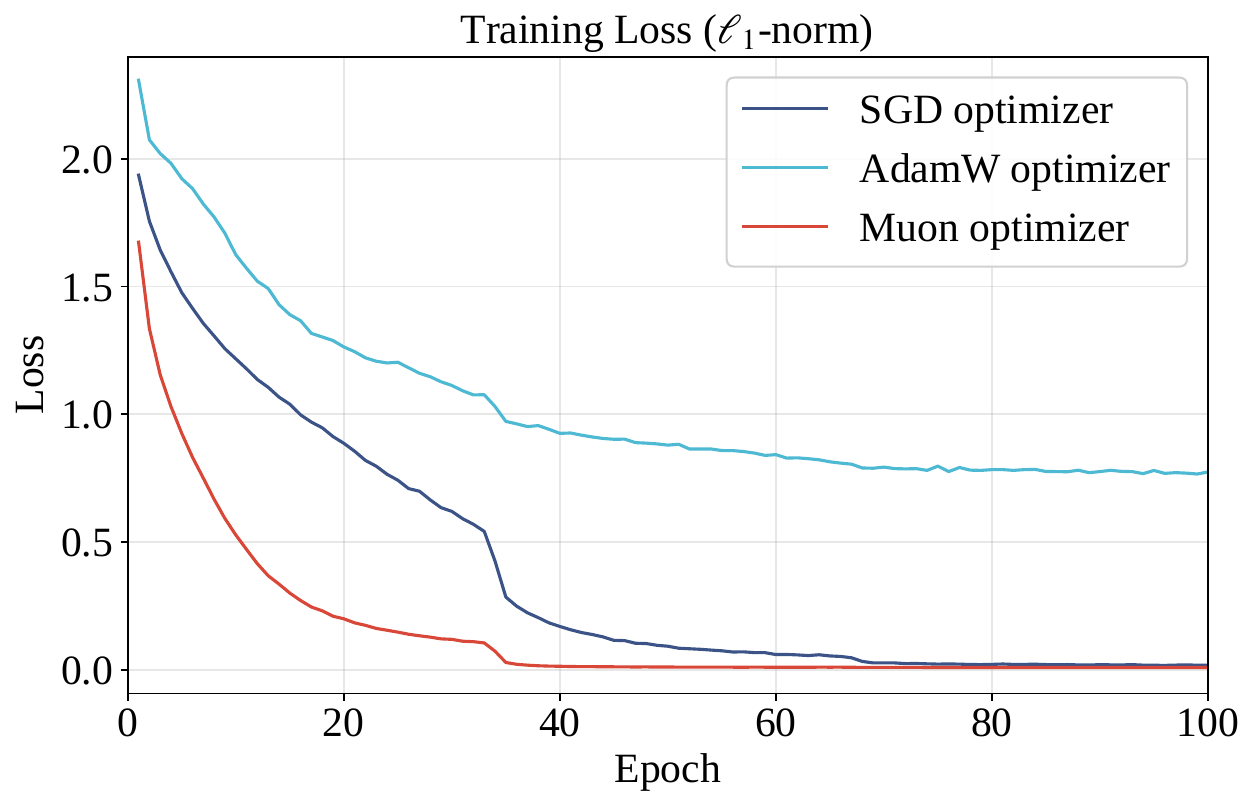}
  \caption{$\ell_1$}
\end{subfigure}\hfill
\begin{subfigure}[t]{0.245\textwidth}
  \centering
  \includegraphics[width=\linewidth]{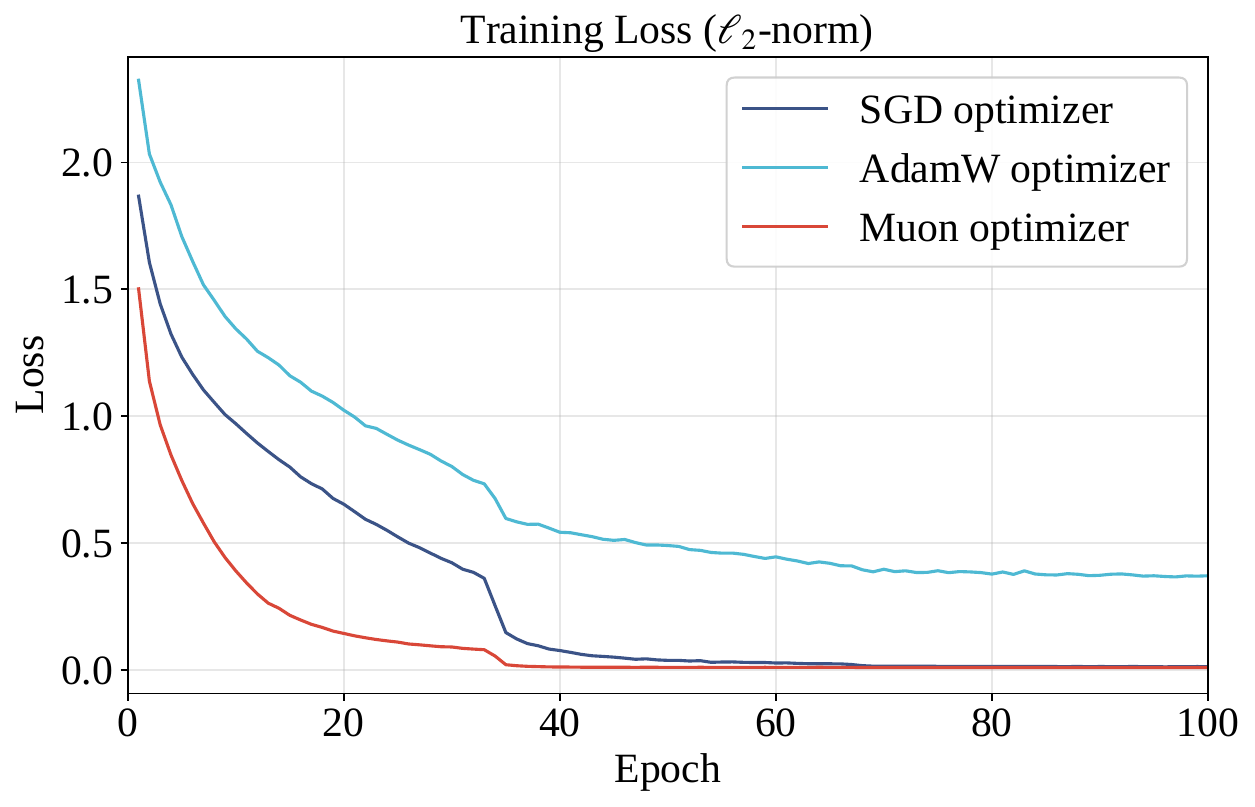}
  \caption{$\ell_2$}
\end{subfigure}
\caption{Training loss on WRN-34-20~\cite{zagoruyko2016wide}.}
\label{fig:loss_wrn_34_20}
\end{figure*}
\begin{figure*}[!t]
\centering
\begin{subfigure}[t]{0.245\textwidth}
  \centering
  \includegraphics[width=\linewidth]{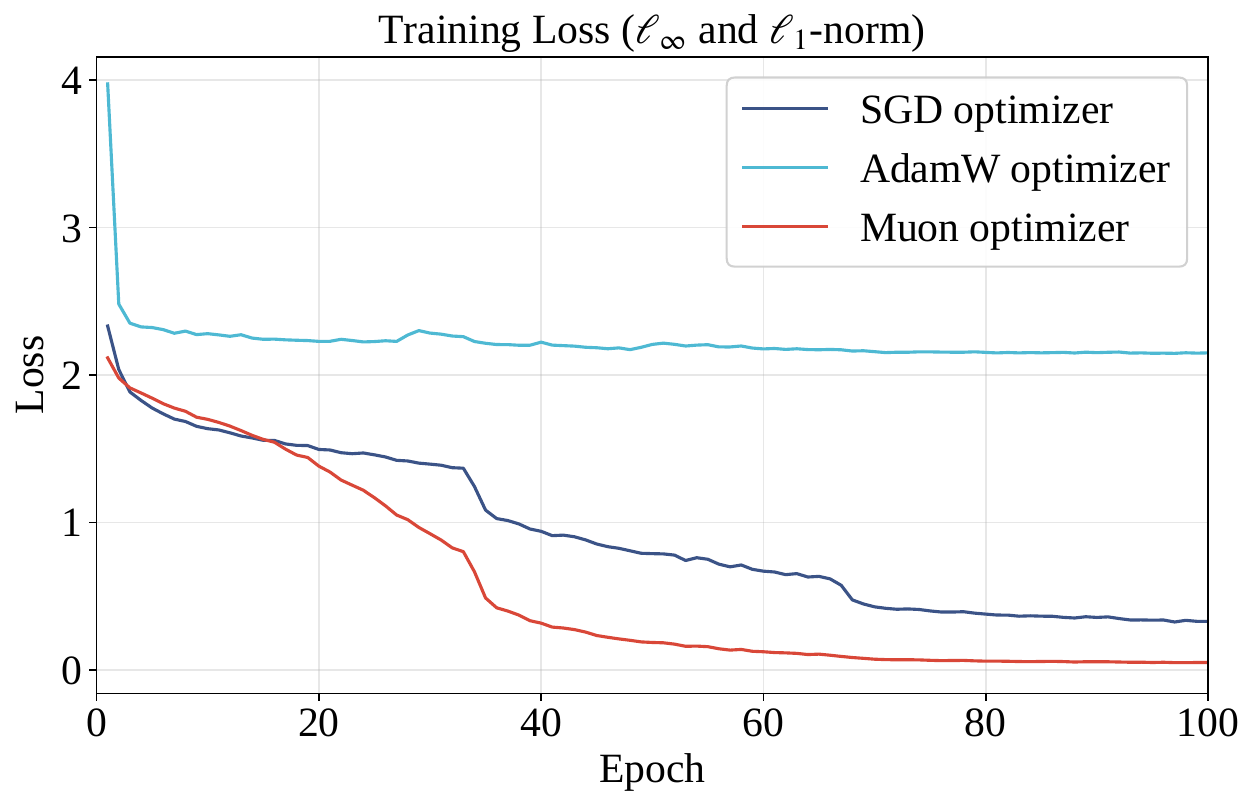}
  \caption{$\ell_\infty + \ell_1$}
\end{subfigure}\hfill
\begin{subfigure}[t]{0.245\textwidth}
  \centering
  \includegraphics[width=\linewidth]{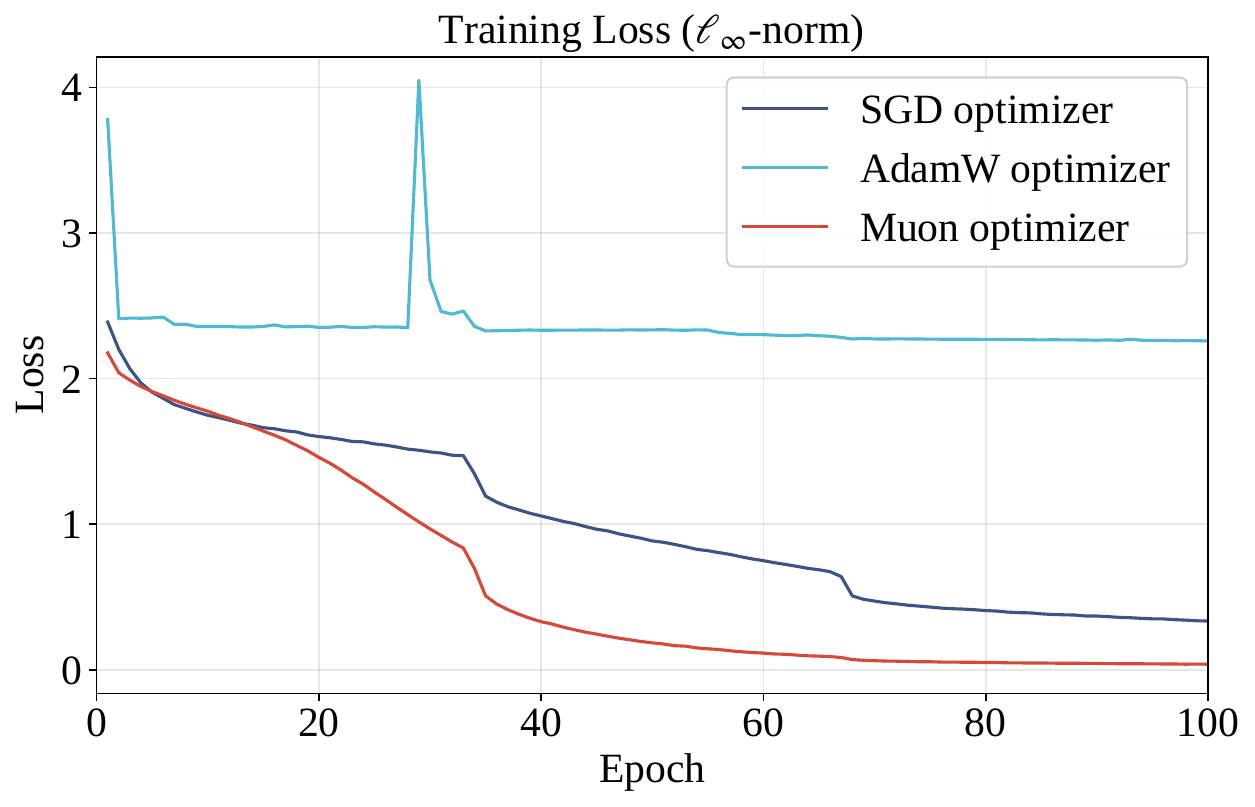}
  \caption{$\ell_\infty$}
\end{subfigure}\hfill
\begin{subfigure}[t]{0.245\textwidth}
  \centering
  \includegraphics[width=\linewidth]{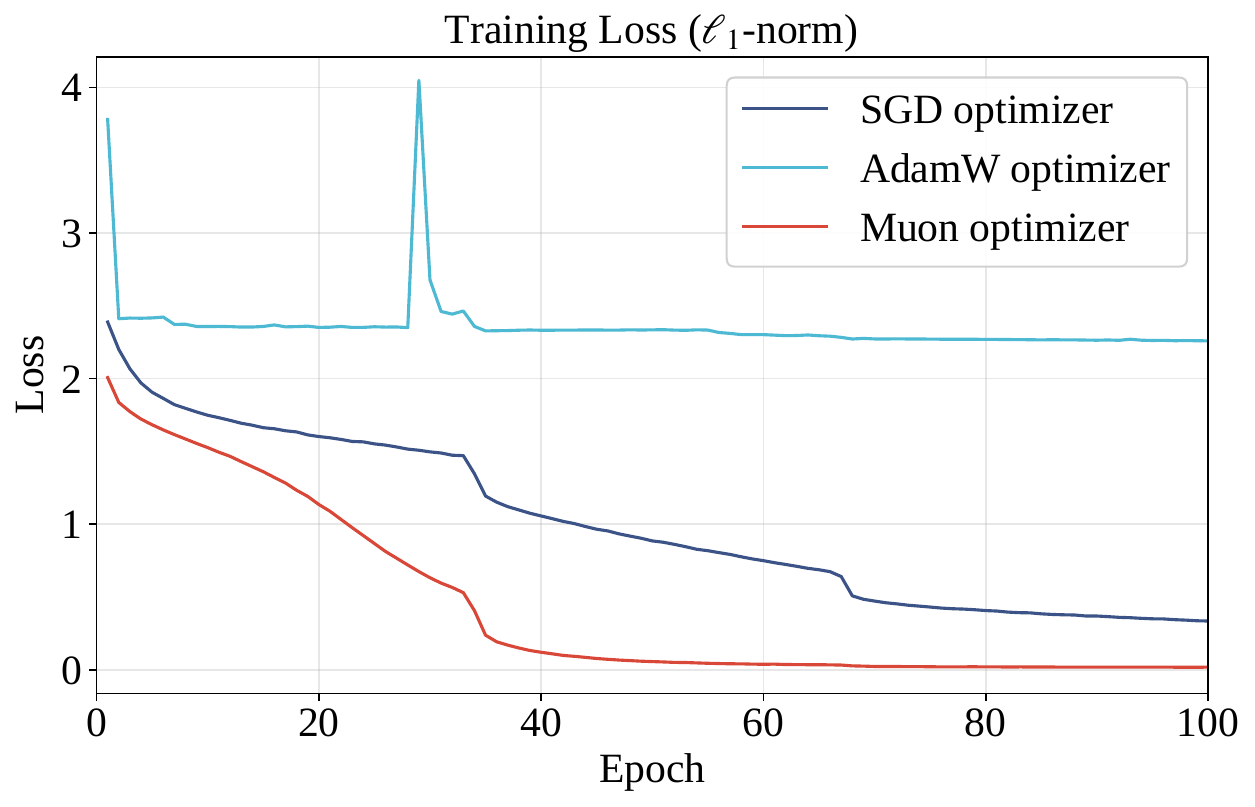}
  \caption{$\ell_1$}
\end{subfigure}\hfill
\begin{subfigure}[t]{0.245\textwidth}
  \centering
  \includegraphics[width=\linewidth]{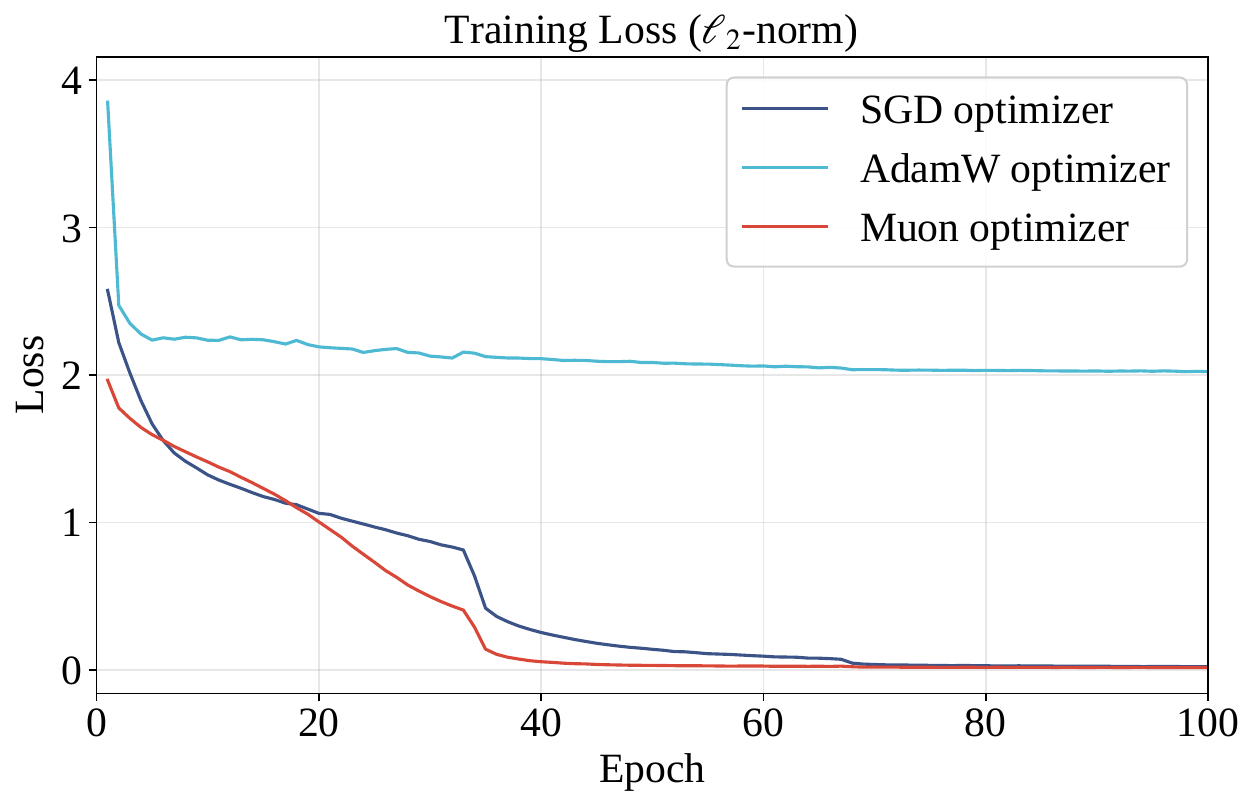}
  \caption{$\ell_2$}
\end{subfigure}
\caption{Training loss on ViT-B~\cite{dosovitskiy2021image}.}
\label{fig:loss_vit_b}
\end{figure*}
\begin{figure*}[!t]
\centering
\begin{subfigure}[t]{0.245\textwidth}
  \centering
  \includegraphics[width=\linewidth]{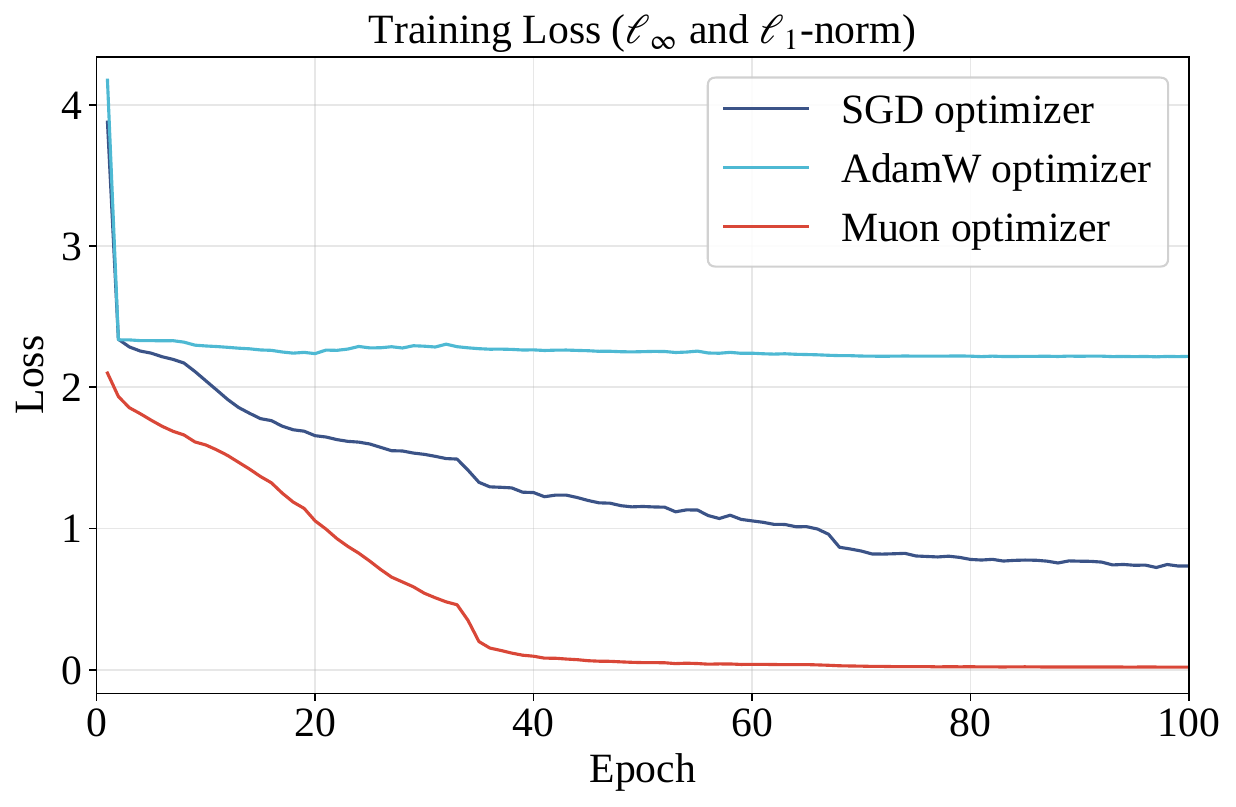}
  \caption{$\ell_\infty + \ell_1$}
\end{subfigure}\hfill
\begin{subfigure}[t]{0.245\textwidth}
  \centering
  \includegraphics[width=\linewidth]{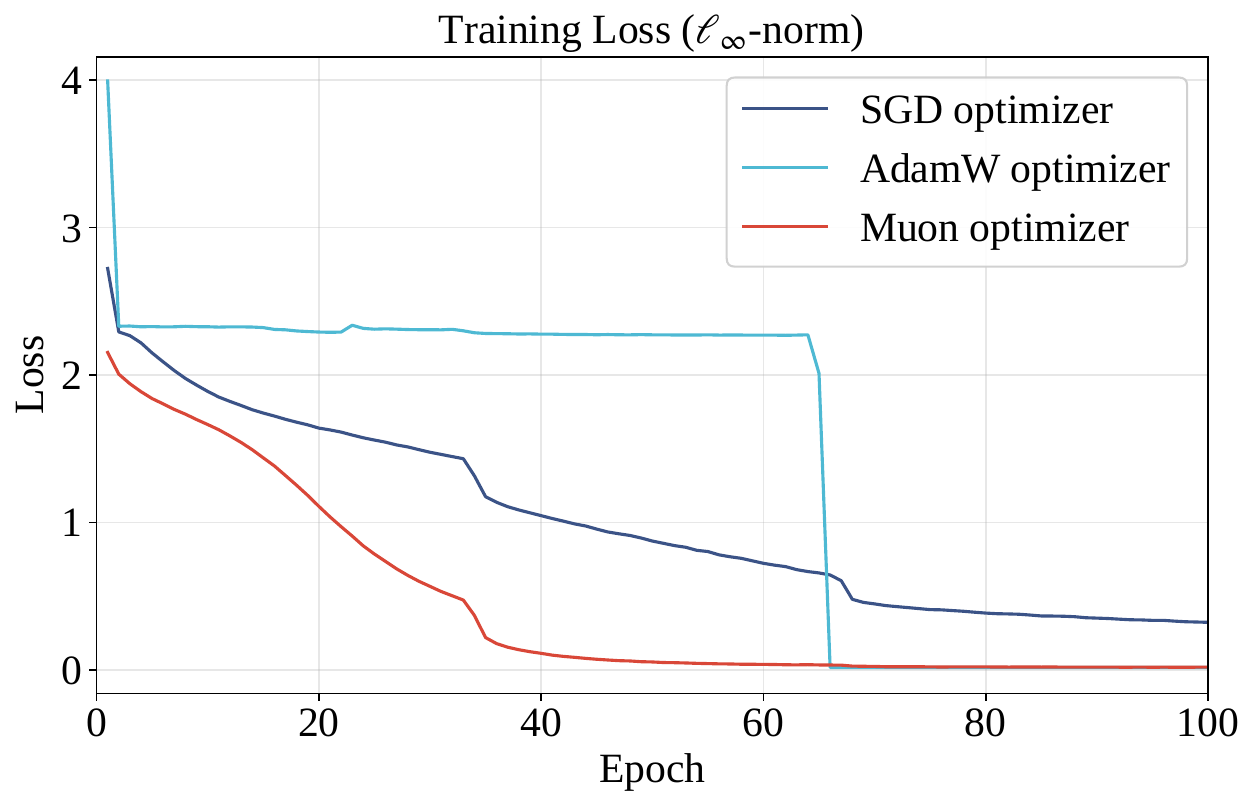}
  \caption{$\ell_\infty$}
\end{subfigure}\hfill
\begin{subfigure}[t]{0.245\textwidth}
  \centering
  \includegraphics[width=\linewidth]{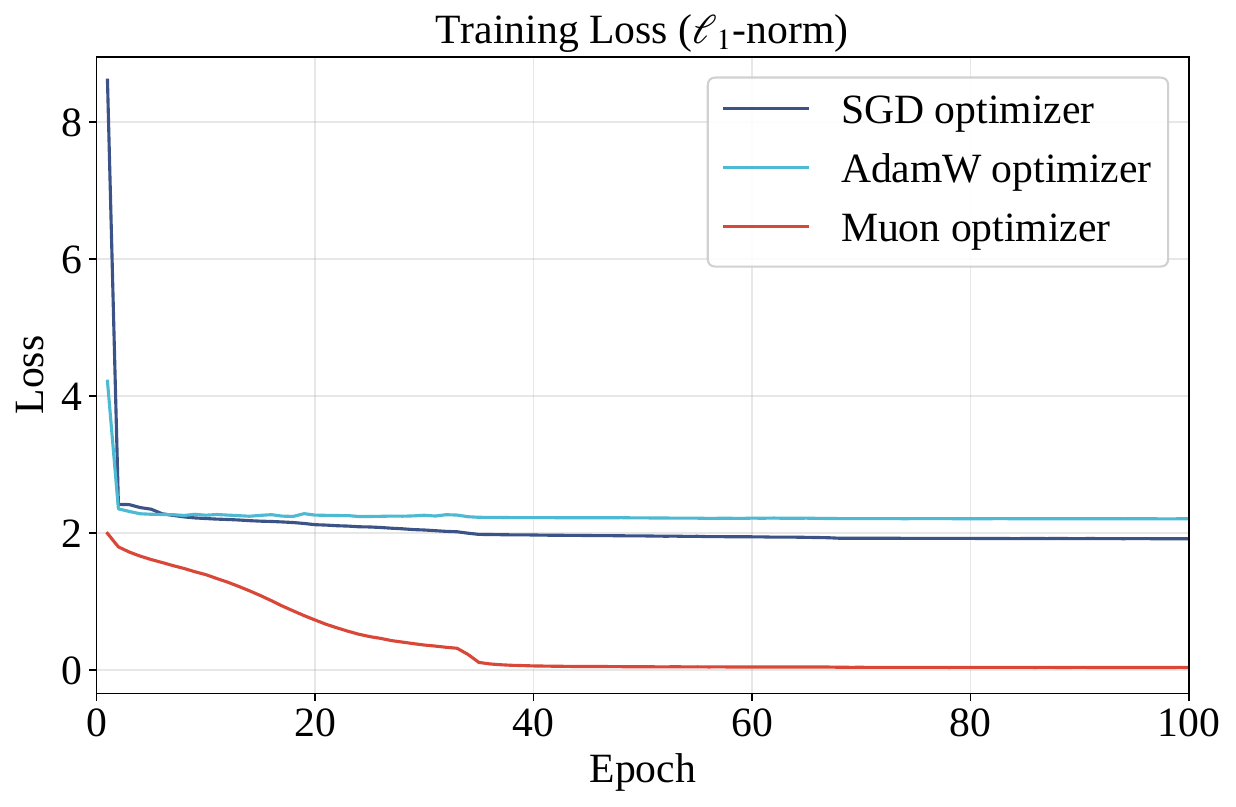}
  \caption{$\ell_1$}
\end{subfigure}\hfill
\begin{subfigure}[t]{0.245\textwidth}
  \centering
  \includegraphics[width=\linewidth]{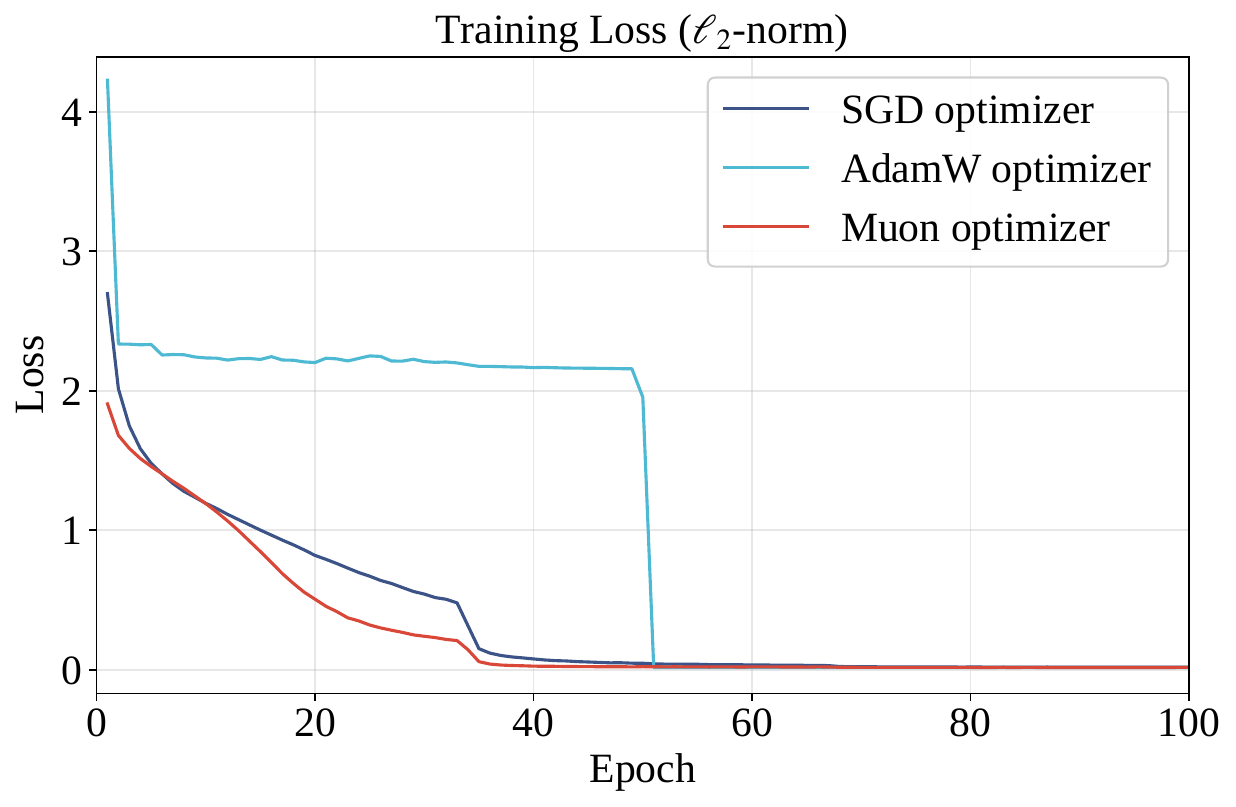}
  \caption{$\ell_2$}
\end{subfigure}
\caption{Training loss on ViT-L~\cite{dosovitskiy2021image}.}
\label{fig:loss_vit_l}
\end{figure*}
\par Under controlled conditions (identical learning rate of 0.01 for all optimizers on PreActResNet-18), Figure~\ref{fig:nuclear_norm_descent} reveals a striking pattern: SGD produces the largest gradient norms, yet Muon achieves the fastest loss descent. This diagnostic suggests that Muon can descend the adversarial objective faster under the same nominal learning rate. However, because different optimizers induce different effective update magnitudes, this experiment should not be interpreted as isolating the update direction alone. Theorem~\ref{thm:adv_loss_refined} explains this: the polar factor $\mathrm{Ortho}(M)$ maximizes the inner product $\langle G, U\rangle$ among all directions $U$ with unit-norm columns, so each Muon step extracts the maximum possible descent per unit of parameter change. In contrast, SGD applies the raw gradient direction, which may not be well-conditioned and waste update magnitude on non-descent directions.
\begin{figure}[!t]
    \centering
    \includegraphics[width=0.6\hsize]{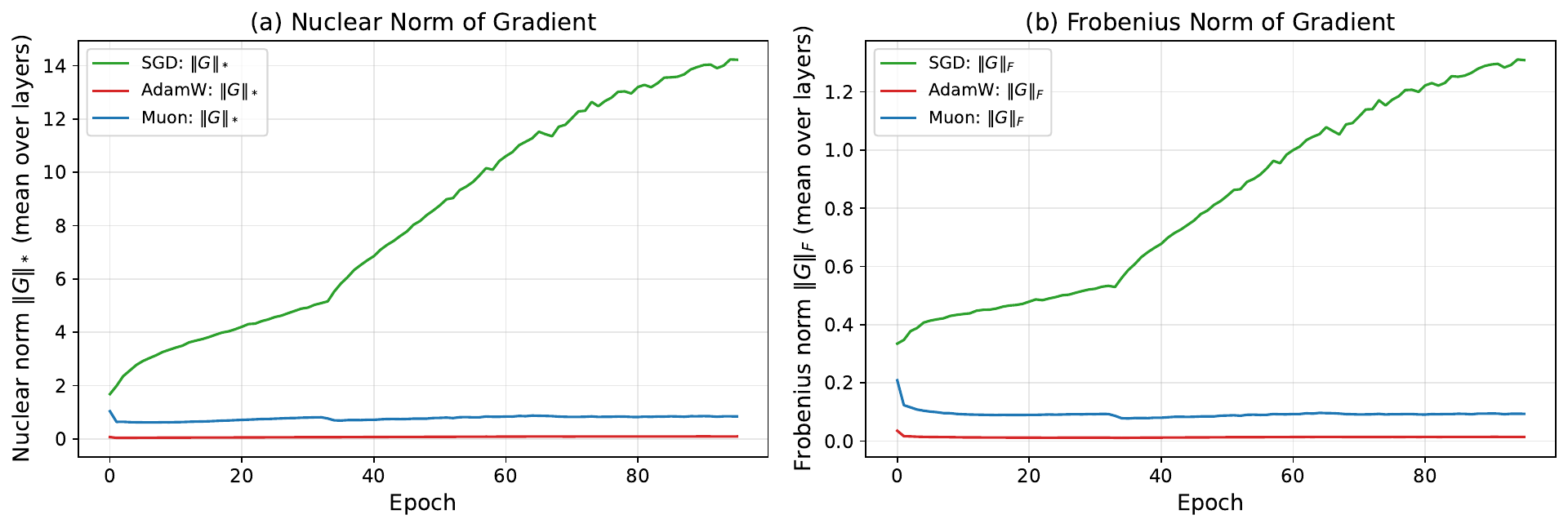}
    \caption{The variation tendency of nuclear gradient descent norms.}
    \label{fig:nuclear_norm_descent}
\end{figure}
\section{Robust Overfitting Diagnostics}
\label{sec:robust_overfitting}
\par The ``robust overfitting'' phenomenon is a concern in adversarial learning. We visualize the clean accuracy, $\ell_\infty$ robust accuracy, and union robust accuracy of PreActResNet-18~\cite{he2016identity}. On CIFAR-10, Figure~\ref{fig:robust_overfitting} illustrates the transition process. As shown in the figures, Muon converges substantially faster than both SGD and AdamW in the early training phase: by epoch 10--15, Muon already achieves about 40--44\% union robust accuracy, while SGD is still around 25--35\% and AdamW around 30--38\%. This rapid early convergence is consistent with our theoretical finding (Theorem~\ref{thm:adv_loss_refined}) that Muon descends the adversarial loss along the gradient nuclear-norm direction, providing stronger per-step progress than vanilla SGD. From a practical standpoint, this means Muon can reach SGD-level robustness in roughly one-third of the training budget, which is a meaningful advantage in compute-constrained settings. In the later phase (epoch 40--100), Muon stabilizes at approximately 40\% ($\ell_\infty$) and 40\% (union), SGD converges to a similar level (about 40\%), while AdamW continues to improve and ultimately reaches about 36\%. The fact that AdamW gradually narrows the gap in the long run is an interesting finding. AdamW's per-parameter adaptive scaling may contribute to this phenomenon. This trade-off between Muon (fast convergence and stable plateau) and AdamW (slow start but higher asymptote) is itself a novel empirical observation about optimizer dynamics under AT. We also note that the oscillation magnitude of SGD (green curve) in the first 35 epochs is substantially larger than Muon's, which supports the paper's claim about Muon providing more stable optimization geometry. This stability, combined with fast early convergence, makes Muon particularly attractive for architectures or settings where training is terminated early or computational budgets are tight.
\begin{figure*}[!t]
\centering
\begin{subfigure}[t]{0.33\textwidth}
  \centering
  \includegraphics[width=\linewidth]{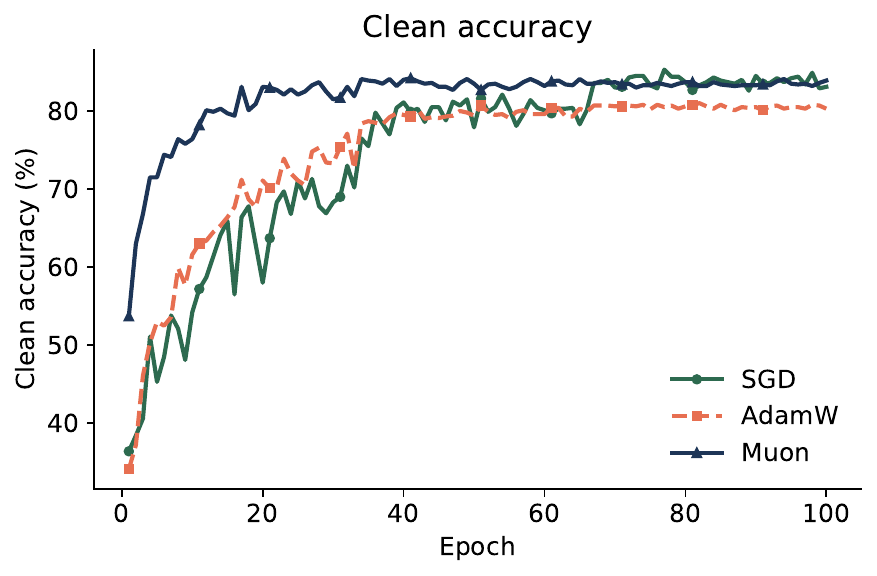}
  \caption{Clean accuracy}
\end{subfigure}\hfill
\begin{subfigure}[t]{0.33\textwidth}
  \centering
  \includegraphics[width=\linewidth]{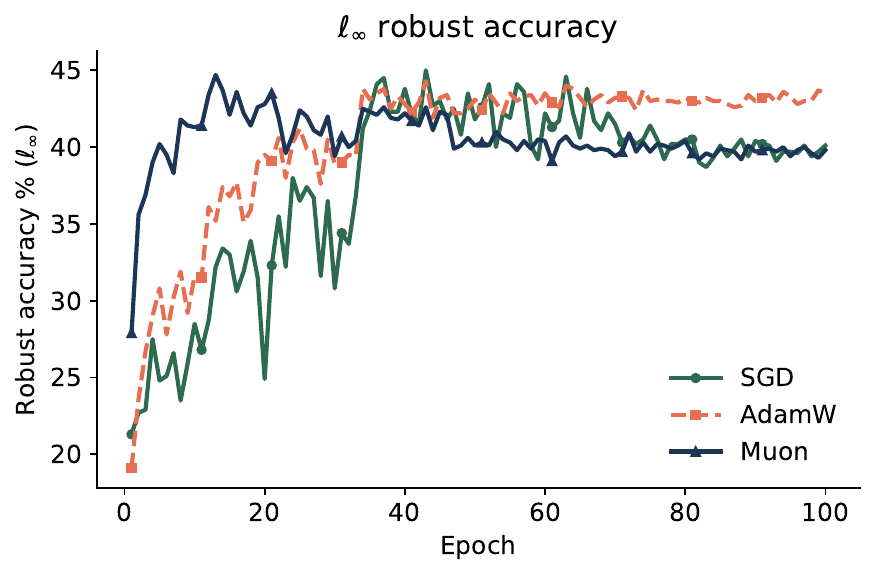}
  \caption{$\ell_\infty$ robust accuracy}
\end{subfigure}\hfill
\begin{subfigure}[t]{0.33\textwidth}
  \centering
  \includegraphics[width=\linewidth]{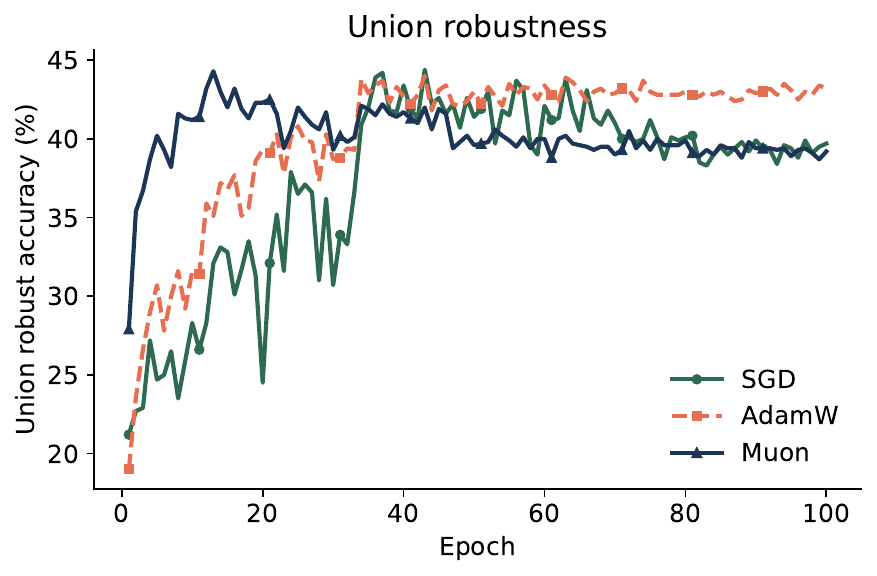}
  \caption{Union robust accuracy}
\end{subfigure}
\caption{Robustness evaluation during AT process.}
\label{fig:robust_overfitting}
\end{figure*}
\par In the main text, we have observed the same ``robust overfitting" phenomenon in several Muon-based AT runs. In particular, the best checkpoint can achieve strong union robustness, whereas the last checkpoint may show a large degradation under the union attack. This indicates that Muon's orthogonalized update geometry improves early and middle-stage robust optimization, but does not by itself eliminate late-stage robust overfitting. This observation refines our interpretation of Muon. The nuclear-norm descent view in Theorem~\ref{thm:adv_loss_refined} explains why Muon can make rapid early progress on the adversarial objective. However, robust overfitting also depends on learning-rate decay, momentum accumulation, finite-sample effects, and the mismatch between the sampled training threat norm and the full union evaluation. Therefore, Muon should be regarded as an optimizer-level mechanism that can accelerate robust optimization, rather than as a standalone robust-overfitting mitigation.
\par For practical use, we recommend robust-validation-based early stopping. Specifically, for multi-norm AT, checkpoints should be selected according to validation union robust accuracy rather than clean accuracy, final-epoch performance, or a single-norm robust metric. This criterion introduces no additional training cost and directly targets the evaluation metric of interest. Other alleviation methods, such as Adversarial Weight Perturbation (AWP)~\cite{wu2020adversarial} and label smoothing~\cite{szegedy2016rethinking}, are complementary to Muon and can be combined with its orthogonalized update geometry in future work.

\section{Additional Empirical Robustness Experiments}
\label{sec:additional_empirical_robustness_experiments}
\subsection{AutoAttack Evaluation}
\par We dive into further studies on empirical robustness to explore additional insights. Although APGD is used as the default evaluation attack for controlled optimizer comparisons, we further conduct stronger evaluations to rule out attack underfitting. Table~\ref{tab:robustness_comparison} compares the defense methods~\cite{croce2022adversarial} on three optimizers with some representative baseline methods, such as AT~\cite{madry2018towards}, FastAT~\cite{wong2020fast} and its variant~\cite{andriushchenko2020understanding}, FreeAT~\cite{shafahi2019adversarial}, and a defense built on Softmax Cross-Entropy Loss~\cite{pang2020rethinking}. The evaluated model is WRN-34-10~\cite{zagoruyko2016wide}. It reports AutoAttack evaluation results under the RobustBench protocol with the threat model $\ell_\infty(\epsilon=8/255)$. The Muon variant achieves 46.71\% robust accuracy under AutoAttack~\cite{croce2020reliable}, compared with 44.11\% for SGD and 42.47\% for AdamW.

\begin{table}[!t]
\centering
\caption{The robustness comparison (RobustBench~\cite{croce2021robustbench}).}
\label{tab:robustness_comparison}
\begin{tabular}{@{}llcc@{}}
\toprule
\textbf{Method} & \textbf{Clean (\%)} & \textbf{AutoAttack (\%)} \\ \midrule
 Defense \cite{croce2022adversarial} + Muon & 83.16 & \textbf{46.71} \\
 Defense \cite{croce2022adversarial} + SGD & 80.02 & 44.11 \\
 Defense \cite{croce2022adversarial} + AdamW & 69.92 & 42.47 \\ \midrule
 AT~\cite{madry2018towards} & 87.14 & 44.04 \\
 FastAT's variant~\cite{andriushchenko2020understanding} & 79.84 & 43.93 \\
 Softmax Cross-Entropy Loss~\cite{pang2020rethinking} & 80.89 & 43.48 \\
 FastAT~\cite{wong2020fast} & 83.34 & 43.21 \\
 FreeAT~\cite{shafahi2019adversarial} & \textbf{86.11} & 41.47 \\ \bottomrule
\end{tabular}
\end{table}
\subsection{Robustness under Stronger White-box Attacks}
\par To rule out attack underfitting and loss-specific overestimation, we evaluate the trained models using stronger white-box attacks, including PGD-100 and CW-100 under both $\ell_\infty$ and $\ell_2$ threat models. These evaluations complement AutoAttack/RobustBench-style testing and provide additional evidence that the observed robustness is not an artifact of the default APGD evaluation. Since Muon affects only the optimization trajectory during training, the deployed classifier remains a standard deterministic differentiable model. Hence, there is no inference-time stochastic transformation over which EOT must average, and no non-differentiable or gradient-obfuscating module for BPDA to approximate~\cite{athalye2018obfuscated,athalye2018synthesizing}. This distinguishes our setting from preprocessing defense~\cite{xu2018feature} or randomization methods~\cite{cohen2019certified}, where adaptive attacks must be tailored to the defense mechanism~\cite{tramer2020adaptive}. We therefore evaluate the resulting deterministic models using strong white-box attacks and AutoAttack protocols~\cite{croce2020reliable}.
\par Table \ref{tab:preactresnet8_white_box_results} checks the stronger white-box attack evaluation results on PreActResNet-18~\cite{he2016identity}. The perturbation budgets are set to $\epsilon_\infty=8/255$ and $\epsilon_2=0.5$ for $\ell_\infty$ and $\ell_2$ attacks, respectively. For each norm, SGD, AdamW, and Muon are compared using best and last checkpoints across PGD-100~\cite{madry2018towards} and C\&W-100~\cite{carlini2017towards} attacks under $\ell_\infty$ and $\ell_2$ settings. Overall, Muon achieves the robust performance in most groups, especially for the last checkpoint, indicating the applicability in the security-sensitive setting.
\begin{table*}[!t]
\centering
\footnotesize
\caption{The white-box attack evaluation results on PreActResNet-18~\cite{he2016identity}.}
\label{tab:preactresnet8_white_box_results}
\setlength{\tabcolsep}{3pt}
\renewcommand{\arraystretch}{1.03}
\begin{tabular}{llcccc@{\hspace{10pt}}cccc}
\toprule
\multirow{2}{*}{Norm} & \multirow{2}{*}{Opt.}
& \multicolumn{4}{c}{Best}
& \multicolumn{4}{c}{Last} \\
\cmidrule(lr){3-6}\cmidrule(l){7-10}
&& PGD ($\ell_\infty$) & C\&W ($\ell_\infty$) & PGD ($\ell_2$) & C\&W ($\ell_2$)
& PGD ($\ell_\infty$) & C\&W ($\ell_\infty$) & PGD ($\ell_2$) & C\&W ($\ell_2$) \\
\midrule

\multirow{3}{*}{$\ell_\infty+\ell_1$}
& SGD   & \textbf{43.44} & \textbf{43.25} & \textbf{66.29} & \textbf{65.32}
        & 37.30 & 38.23 & 64.96 & 64.67 \\
& AdamW & 39.24 & 39.55 & 56.56 & 55.67
        & 24.74 & 24.89 & 57.87 & 56.12 \\
& Muon  & 42.76 & 42.10 & 65.25 & 63.56
        & \textbf{39.97} & \textbf{40.89} & \textbf{68.09} & \textbf{67.78} \\

\midrule

\multirow{3}{*}{$\ell_\infty$}
& SGD   & \textbf{48.07} & \textbf{48.11} & 60.00 & \textbf{59.19}
        & 43.78 & 44.19 & 58.54 & 58.44 \\
& AdamW & 41.76 & 40.95 & 49.55 & 48.53
        & 36.70 & 35.48 & 51.64 & 50.15 \\
& Muon  & 47.03 & 45.66 & \textbf{61.15} & \textbf{59.19}
        & \textbf{45.12} & \textbf{45.78} & \textbf{62.07} & \textbf{61.97} \\

\midrule

\multirow{3}{*}{$\ell_1$}
& SGD   & 19.02 & 19.51 & \textbf{59.44} & \textbf{59.95}
        & 19.11 & 19.62 & 59.53 & 60.01 \\
& AdamW & 24.28 & 26.39 & 55.62 & 55.19
        & \textbf{42.27} & \textbf{40.65} & 56.04 & 55.82 \\
& Muon  & \textbf{25.97} & \textbf{27.03} & 59.39 & 58.32
        & 19.76 & 20.81 & \textbf{63.67} & \textbf{64.23} \\

\midrule

\multirow{3}{*}{$\ell_2$}
& SGD   & 27.15 & 28.94 & \textbf{65.95} & \textbf{66.32}
        & 21.27 & 23.46 & 63.58 & 64.02 \\
& AdamW & 25.99 & 28.48 & 59.71 & 59.97
        & 24.84 & \textbf{27.62} & 59.82 & 60.38 \\
& Muon  & \textbf{29.88} & \textbf{31.71} & 65.23 & 64.75
        & \textbf{25.12} & 25.80 & \textbf{67.06} & \textbf{67.40} \\
\bottomrule
\end{tabular}
\end{table*}
\par Table \ref{tab:wrn_34_10_white_box_results} reports white-box robustness on WRN-34-10~\cite{zagoruyko2016wide} under PGD-100 and C\&W-100 attacks for both $\ell_\infty$ and $\ell_2$ threat models.  Across four training norms, Muon consistently delivers the strongest or near-strongest robustness, particularly at the last checkpoint. The gains are most pronounced under $\ell_1$ and $\ell_2$ training, suggesting that Muon is beneficial for robust generalization.
\begin{table*}[!t]
\centering
\footnotesize
\caption{The white-box attack evaluation results on WRN-34-10~\cite{zagoruyko2016wide}.}
\label{tab:wrn_34_10_white_box_results}
\setlength{\tabcolsep}{3pt}
\renewcommand{\arraystretch}{1.03}
\begin{tabular}{llcccc@{\hspace{10pt}}cccc}
\toprule
\multirow{2}{*}{Norm} & \multirow{2}{*}{Opt.}
& \multicolumn{4}{c}{Best}
& \multicolumn{4}{c}{Last} \\
\cmidrule(lr){3-6}\cmidrule(l){7-10}
&& PGD ($\ell_\infty$) & C\&W ($\ell_\infty$) & PGD ($\ell_2$) & C\&W ($\ell_2$)
& PGD ($\ell_\infty$)& C\&W ($\ell_\infty$) & PGD ($\ell_2$) & C\&W ($\ell_2$) \\
\midrule

\multirow{3}{*}{$\ell_\infty+\ell_1$}
& SGD   & 40.48 & 40.60 & \textbf{64.59} & \textbf{63.73}
        & 33.30 & 33.89 & 61.91 & 61.98 \\
& AdamW & 39.47 & 40.64 & 61.38 & 61.05
        & 37.89 & 39.02 & 62.80 & 62.44 \\
& Muon  & \textbf{41.98} & \textbf{41.14} & 60.19 & 58.93
        & \textbf{39.43} & \textbf{39.69} & \textbf{65.35} & \textbf{65.48} \\

\midrule

\multirow{3}{*}{$\ell_\infty$}
& SGD   & 46.53 & 46.35 & 59.66 & 58.65
        & 38.68 & 39.34 & 54.03 & 54.25 \\
& AdamW & 44.56 & 43.93 & 51.65 & 50.63
        & 44.21 & 43.92 & 50.87 & 49.89 \\
& Muon  & \textbf{49.41} & \textbf{48.80} & \textbf{60.48} & \textbf{59.53}
        & \textbf{44.81} & \textbf{45.05} & \textbf{54.72} & \textbf{55.12} \\

\midrule

\multirow{3}{*}{$\ell_1$}
& SGD   & 22.24 & 21.26 & 38.47 & 37.33
        & 18.85 & 18.89 & 58.27 & 58.55 \\
& AdamW & 25.42 & 26.10 & 52.07 & 51.42
        & \textbf{23.29} & \textbf{26.12} & 56.51 & 56.57 \\
& Muon  & \textbf{27.20} & \textbf{28.89} & \textbf{60.53} & \textbf{59.87}
        & 22.50 & 22.77 & \textbf{63.32} & \textbf{63.48} \\

\midrule

\multirow{3}{*}{$\ell_2$}
& SGD   & 24.10 & 25.48 & 55.21 & 54.72
        & 23.47 & 23.44 & 62.58 & 62.77 \\
& AdamW & 25.54 & 27.19 & 53.40 & 53.03
        & 23.75 & 25.18 & 61.53 & 62.36 \\
& Muon  & \textbf{28.77} & \textbf{28.96} & \textbf{68.33} & \textbf{68.59}
        & \textbf{28.85} & \textbf{29.04} & \textbf{68.76} & \textbf{68.99} \\

\bottomrule
\end{tabular}
\end{table*}
\par Table \ref{tab:wrn_34_20_white_box_results} presents white-box robustness results on WRN-34-20~\cite{zagoruyko2016wide} under PGD and C\&W attacks with $\ell_\infty$ and $\ell_2$ constraints. Across all training norms, Muon achieves the best or comparable robustness performance, with especially strong gains under $\ell_2$ training. Its last-checkpoint results remain consistently competitive, suggesting improved robustness stability compared with SGD and AdamW.
\begin{table*}[t]
\centering
\footnotesize
\caption{The white-box attack evaluation results on WRN-34-20~\cite{zagoruyko2016wide}.}
\label{tab:wrn_34_20_white_box_results}
\setlength{\tabcolsep}{3pt}
\renewcommand{\arraystretch}{1.03}
\begin{tabular}{llcccc@{\hspace{10pt}}cccc}
\toprule
\multirow{2}{*}{Norm} & \multirow{2}{*}{Opt.}
& \multicolumn{4}{c}{Best}
& \multicolumn{4}{c}{Last} \\
\cmidrule(lr){3-6}\cmidrule(l){7-10}
&& PGD ($\ell_\infty$) & C\&W ($\ell_\infty$) & PGD ($\ell_2$) & C\&W ($\ell_2$)
& PGD ($\ell_\infty$) & C\&W ($\ell_\infty$) & PGD ($\ell_2$) & C\&W ($\ell_2$) \\
\midrule

\multirow{3}{*}{$\ell_\infty+\ell_1$}
& SGD   & 40.58 & 39.12 & 62.14 & 60.47
        & 35.23 & 35.53 & 62.25 & 62.42 \\
& AdamW & 38.34 & 37.86 & 59.50 & 58.12
        & 38.34 & 37.86 & 59.49 & 58.12 \\
& Muon  & \textbf{43.03} & \textbf{42.41} & \textbf{63.19} & \textbf{61.66}
        & \textbf{41.78} & \textbf{42.07} & \textbf{62.82} & \textbf{62.96} \\

\midrule

\multirow{3}{*}{$\ell_\infty$}
& SGD   & 45.13 & 43.08 & 57.92 & 55.81
        & 40.13 & 40.47 & 54.90 & 55.36 \\
& AdamW & 44.81 & 45.05 & 50.96 & 50.22
        & 44.17 & 43.17 & 50.60 & 49.25 \\
& Muon  & \textbf{49.71} & \textbf{48.56} & \textbf{60.39} & \textbf{59.07}
        & \textbf{47.71} & \textbf{48.01} & \textbf{57.83} & \textbf{58.08} \\

\midrule

\multirow{3}{*}{$\ell_1$}
& SGD   & 23.91 & 23.80 & 43.44 & 42.51
        & 20.04 & 20.20 & 60.65 & 60.79 \\
& AdamW & 23.79 & \textbf{26.23} & 53.03 & 52.74
        & 21.26 & 22.64 & 57.92 & 59.11 \\
& Muon  & \textbf{24.20} & 24.43 & \textbf{65.88} & \textbf{66.05}
        & \textbf{23.65} & \textbf{23.74} & \textbf{65.36} & \textbf{65.44} \\

\midrule

\multirow{3}{*}{$\ell_2$}
& SGD   & 26.23 & 26.38 & 64.42 & 64.66
        & 25.67 & 25.78 & 65.06 & 65.22 \\
& AdamW & 26.83 & 26.58 & 61.62 & 62.95
        & 27.53 & 26.52 & 61.33 & 62.56 \\
& Muon  & \textbf{30.71} & \textbf{30.88} & \textbf{70.37} & \textbf{70.49}
        & \textbf{30.29} & \textbf{30.50} & \textbf{70.55} & \textbf{70.69} \\

\bottomrule
\end{tabular}
\end{table*}
\subsection{Robustness under Black-box Attacks}
\par Black-box transfer attacks provide a practical measure of robustness beyond model-specific white-box optimization. In realistic deployments, attackers may not know the exact architecture, parameters, or defense pipeline, but can still craft adversarial examples on surrogate models and transfer them to the target. Therefore, evaluating black-box adversarial examples tests whether robustness generalizes across attack sources rather than merely resisting a known gradient path. Strong black-box performance indicates reduced dependence on obfuscated gradients and better security under incomplete information. This evaluation is thus essential for assessing deployable robustness, transfer resistance, and the reliability of adversarial training in real-world threat models and systems. We test the robustness performance under the transfer-based black-box attacks. 
\par We first evaluate a naturally trained PreActResNet-18~\cite{he2016identity} under adversarial attacks. Its robust accuracies under PGD-100 ($\ell_\infty$), CW-100 ($\ell_\infty$), PGD-100 ($\ell_2$), and CW-100 ($\ell_2$) are 2.34\%, 2.85\%, 16.73\%, and 16.22\%, respectively. Then, we transfer the adversarial examples to WRN-34-10~\cite{zagoruyko2016wide}. Table \ref{tab:wrn_34_10_blackbox_transfer} illustrates the robustness evaluation results that the performance of the models trained by Muon is considerable.
\begin{table*}[!t]
\centering
\footnotesize
\caption{The black-box transfer attack results on WRN-34-10~\cite{zagoruyko2016wide}.}
\label{tab:wrn_34_10_blackbox_transfer}
\setlength{\tabcolsep}{3pt}
\renewcommand{\arraystretch}{1.03}
\begin{tabular}{llcccc@{\hspace{10pt}}cccc}
\toprule
\multirow{2}{*}{Norm} & \multirow{2}{*}{Opt.}
& \multicolumn{4}{c}{Best}
& \multicolumn{4}{c}{Last} \\
\cmidrule(lr){3-6}\cmidrule(l){7-10}
&& PGD ($\ell_\infty$) & C\&W ($\ell_\infty$) & PGD ($\ell_2$) & C\&W ($\ell_2$)
& PGD ($\ell_\infty$) & C\&W ($\ell_\infty$) & PGD ($\ell_2$) & C\&W ($\ell_2$) \\
\midrule

\multirow{3}{*}{$\ell_\infty+\ell_1$}
& SGD   & \textbf{74.56} & \textbf{74.32} & \textbf{77.06} & \textbf{76.97}
        & 77.97 & 78.08 & 80.20 & 80.19 \\
& AdamW & 71.76 & 71.58 & 74.17 & 74.05
        & 74.39 & 74.34 & 77.35 & 77.30 \\
& Muon  & 68.68 & 68.45 & 71.80 & 71.72
        & \textbf{82.06} & \textbf{81.96} & \textbf{84.14} & \textbf{84.03} \\

\midrule

\multirow{3}{*}{$\ell_\infty$}
& SGD   & 76.38 & 76.27 & 78.52 & 78.43
        & 79.48 & 79.32 & 81.23 & 81.29 \\
& AdamW & 66.64 & 66.58 & 68.37 & 68.40
        & 68.44 & 68.63 & 70.66 & 70.78 \\
& Muon  & \textbf{79.83} & \textbf{79.66} & \textbf{81.72} & \textbf{81.65}
        & \textbf{83.40} & \textbf{83.44} & \textbf{84.96} & \textbf{85.06} \\

\midrule

\multirow{3}{*}{$\ell_1$}
& SGD   & 41.90 & 41.69 & 45.52 & 45.57
        & 79.50 & 79.22 & 82.61 & 82.46 \\
& AdamW & 59.98 & 59.85 & 63.61 & 63.66
        & 68.60 & 68.50 & 73.46 & 73.60 \\
& Muon  & \textbf{72.50} & \textbf{71.71} & \textbf{76.91} & \textbf{76.90}
        & \textbf{83.30} & \textbf{82.90} & \textbf{86.66} & \textbf{86.60} \\

\midrule

\multirow{3}{*}{$\ell_2$}
& SGD   & 70.98 & 70.94 & 74.47 & 74.29
        & 83.82 & 83.93 & 86.29 & 86.21 \\
& AdamW & 66.00 & 65.94 & 70.75 & 70.88
        & 81.81 & 81.66 & 84.65 & 84.61 \\
& Muon  & \textbf{86.86} & \textbf{86.74} & \textbf{89.07} & \textbf{88.98}
        & \textbf{86.86} & \textbf{86.71} & \textbf{89.19} & \textbf{89.10} \\

\bottomrule
\end{tabular}
\end{table*}
\par Similar settings are implemented on WRN-34-20~\cite{zagoruyko2016wide}. Table \ref{tab:wrn_34_20_blackbox_transfer} reports black-box transfer robustness on WRN-34-20 (PreActResNet-18 as the source model). Muon consistently outperforms SGD and AdamW in nearly all settings, with especially large margins under $\ell_1$ and $\ell_2$ training. Notably, Muon maintains strong last-checkpoint robustness, achieving 85.10/84.96 under $\ell_\infty$ transfer attacks for $\ell_\infty$ training and 90.35/90.32 under $\ell_2$ transfer attacks for $\ell_2$ training. These results indicate that Muon’s gains are not limited to white-box optimization paths, but generalize to transferred adversarial examples, suggesting stronger cross-model robustness and better resistance to real black-box threat models.
\begin{table*}[t]
\centering
\footnotesize
\caption{The black-box transfer attack results on WRN-34-20~\cite{zagoruyko2016wide}.}
\label{tab:wrn_34_20_blackbox_transfer}
\setlength{\tabcolsep}{3pt}
\renewcommand{\arraystretch}{1.03}
\begin{tabular}{llcccc@{\hspace{10pt}}cccc}
\toprule
\multirow{2}{*}{Norm} & \multirow{2}{*}{Opt.}
& \multicolumn{4}{c}{Best}
& \multicolumn{4}{c}{Last} \\
\cmidrule(lr){3-6}\cmidrule(l){7-10}
&& PGD ($\ell_\infty$) & C\&W ($\ell_\infty$) & PGD ($\ell_2$) & C\&W ($\ell_2$)
& PGD ($\ell_\infty$) & C\&W ($\ell_\infty$) & PGD ($\ell_2$) & C\&W ($\ell_2$) \\
\midrule

\multirow{3}{*}{$\ell_\infty+\ell_1$}
& SGD   & 70.67 & 70.51 & 73.52 & 73.61
        & 79.19 & 78.90 & 81.37 & 81.38 \\
& AdamW & 68.34 & 68.23 & 70.95 & 70.98
        & 68.34 & 68.23 & 70.95 & 70.97 \\
& Muon  & \textbf{72.05} & \textbf{71.79} & \textbf{75.00} & \textbf{74.98}
        & \textbf{84.35} & \textbf{84.13} & \textbf{85.73} & \textbf{85.54} \\

\midrule

\multirow{3}{*}{$\ell_\infty$}
& SGD   & 72.19 & 71.94 & 74.39 & 74.30
        & 80.72 & 80.61 & 82.30 & 82.38 \\
& AdamW & 69.47 & 69.45 & 71.50 & 71.47
        & 71.23 & 71.31 & 73.35 & 73.39 \\
& Muon  & \textbf{78.82} & \textbf{78.48} & \textbf{80.66} & \textbf{80.56}
        & \textbf{85.10} & \textbf{84.96} & \textbf{86.48} & \textbf{86.42} \\

\midrule

\multirow{3}{*}{$\ell_1$}
& SGD   & 49.58 & 49.40 & 54.09 & 54.06
        & 80.49 & 80.19 & 83.99 & 83.96 \\
& AdamW & 62.73 & 62.49 & 66.79 & 66.85
        & 76.56 & 76.20 & 80.46 & 80.38 \\
& Muon  & \textbf{84.49} & \textbf{84.15} & \textbf{87.64} & \textbf{87.64}
        & \textbf{84.15} & \textbf{83.76} & \textbf{87.49} & \textbf{87.52} \\

\midrule

\multirow{3}{*}{$\ell_2$}
& SGD   & 84.78 & 84.56 & 87.20 & 87.03
        & 85.13 & 84.70 & 87.29 & 87.36 \\
& AdamW & 81.56 & 81.56 & 84.02 & 83.92
        & 81.11 & 81.03 & 83.82 & 83.67 \\
& Muon  & \textbf{88.19} & \textbf{88.11} & \textbf{90.26} & \textbf{90.21}
        & \textbf{88.30} & \textbf{88.15} & \textbf{90.35} & \textbf{90.32} \\

\bottomrule
\end{tabular}
\end{table*}
\clearpage
\end{document}